\documentclass[12pt]{article}
\usepackage[letterpaper,top=1in,bottom=1in,left=1in,right=1in]{geometry}

\usepackage{lipsum}
\usepackage{amsfonts}
\usepackage{graphicx}
\usepackage{epstopdf}
\usepackage{algorithmic}
\usepackage{algorithm}
\usepackage{subcaption}
\ifpdf
  \DeclareGraphicsExtensions{.eps,.pdf,.png,.jpg}
\else
  \DeclareGraphicsExtensions{.eps}
\fi

\usepackage{enumerate,amsfonts, graphicx, amsmath,bm, xcolor, float}
\usepackage{amsthm}
\usepackage{url}
\usepackage{enumitem}
\setlist[enumerate]{leftmargin=.5in}
\setlist[itemize]{leftmargin=.5in}
\usepackage{xr}
\usepackage{hyperref}

\newcommand\x{\bm{x}}
\renewcommand\a{\bm{a}}

\newcommand\w{\bm{w}}
\renewcommand\c{\bm{c}}
\newcommand\y{\bm{y}}
\newcommand\z{\bm{z}}
\newcommand\q{\bm{q}}

\renewcommand\v{\bm{v}}

\newcommand\h{\bm{h}}

\newcommand\bz{\bm{0}}
\newcommand\bepsilon{\bm{\epsilon}}
\newcommand\bdelta{\bm{\delta}}

\newcommand\bbeta{\bm{\beta}}

\newcommand\bomega{\bm{\omega}}
\newcommand\bsigma{\bm{\sigma}}
\newcommand\R{\bm{\mathrm{R}}}
\newcommand{\norm}[1] {\left \| #1 \right \|} 
\newcommand{\abrack}[1] {\left \langle #1 \right \rangle} 

\newtheorem{agglomeration}{Agglomeration Property}
\newtheorem{theorem}{Theorem}[section]

\newtheorem{lemma}[theorem]{Lemma}
\newtheorem{corollary}[theorem]{Corollary}

\title{Certifying clusters from sum-of-norms clustering}

\author{Tao Jiang\thanks{School of Operations Research and Information Engineering, Cornell University, Ithaca, New York 14850, USA, {\tt tj293@cornell.edu}.} \and
Stephen Vavasis\thanks{Department of Combinatorics \&  Optimization, University of Waterloo, Waterloo, Ontario,  Canada N2L 3G1, {\tt vavasis@uwaterloo.ca}. }}

\usepackage{amsopn}

\usepackage{amsopn}
\ifpdf
\hypersetup{
  pdftitle={ex_article},
  pdfauthor={T. Jiang, and S. Vavasis}
}
\fi

\begin{document}
\maketitle

\begin{abstract}
  Sum-of-norms clustering is a clustering formulation based on convex optimization that automatically induces hierarchy. Multiple algorithms have been proposed to solve the optimization problem: subgradient descent by Hocking et al.\ \cite{hocking}, ADMM and ADA by Chi and Lange\ \cite{Chi}, stochastic incremental algorithm by Panahi et al.\ \cite{Panahi} and semismooth Newton-CG augmented Lagrangian method by Sun et al.\ \cite{dsun1}. All algorithms yield approximate solutions, even though an exact solution is demanded to determine the correct cluster assignment. The purpose of this paper is to close the gap between the output from existing algorithms and the exact solution to the optimization problem. We present a clustering test that identifies and certifies the correct cluster assignment from an approximate solution yielded by any primal-dual algorithm.

  Our certification validates clustering for both unit and multiplicative weights. The test may not succeed if the approximation is inaccurate. However, we show the correct cluster assignment is guaranteed to be certified
  by a primal-dual path following algorithm after sufficiently many iterations, provided that the model parameter $\lambda$ avoids a finite number of bad values.
  Numerical experiments are conducted on Gaussian mixture and half-moon data, which indicate that carefully chosen multiplicative weights increase the recovery power of sum-of-norms clustering.

\end{abstract}

\textbf{keywords:}
  Sum-of-norms clustering, second-order cone programming, finite termination, stopping criteria, duality, strict complementarity


\section{Introduction}
\label{sec:intro}
Clustering is a fundamental problem in unsupervised learning. The goal of clustering is to seek a partition of $n$ points, $\a_1,\a_2, \ldots,\a_n \in \mathbb R^d$, such that points in the same subset are closer to each other than those that are not. Clustering is usually formulated as a discrete optimization problem, which is combinatorially hard to solve and beset by nonoptimal local minimizers. Classical methods such as k-means and hierarchical clustering are prone to these issues.

Meanwhile, issues of hardness and suboptimality of many nonconvex optimization problems are resolved by convex relaxation. At an affordable computational cost, convex relaxation yields a good solution to the original problem. Pelckmans et al.\ \cite{pelckmans}, Hocking et al.\ \cite{hocking}, and Lindsten et al.\ \cite{lindsten} propose the following convex formulation for the clustering problem:
\begin{equation}
    \min_{\x_1,\ldots,\x_n\in\R^d} f'(\x) = \frac{1}{2}\sum_{i=1}^n \norm{\x_i-\a_i}^2 +\lambda\sum_{1 \le i < j \le n}\norm{\x_i-\x_j},
    \label{eq:son-clustering}
\end{equation}
where $\a_1, \a_2, \ldots, \a_n$ denote the given data and $\lambda$ denotes the tuning parameter. The formulation \eqref{eq:son-clustering} is best known as sum-of-norms (SON) clustering, convex clustering, or clusterpath clustering. The clusters are read from the optimizer of \eqref{eq:son-clustering}. Let $\x_1^*, \x_2^*, \ldots, \x_n^*$ denote the optimizer. Points $i, i'$ are assigned to the same cluster if $\x_i^* = \x_{i'}^*$, and they are assigned to different clusters otherwise. In an iterative algorithm, $\x_i^*$ is never known precisely, so an approximation to this test is required,
that is, a means to certify that $\x_i^* = \x_j^*$ given only an approximate solution. We provide more details below.

The first term of the objective function ensures $\x^*$ is a good approximation of the original data $\a$, while the second term penalizes the differences $\x^*_i - \x^*_{i'}$. As a result, the second term tends to make $\x_i^*$ equal to each other for many $i$. In this paper, we only consider the $l_2$ norm.

The purpose of this paper is to present our clustering test which certifies the clusters from sum-of-norms clustering. We also justify the test rigorously. Our clustering test takes a primal and dual feasible solution for the second-order cone formulation of sum-of-norms clustering and a clustering determined by the user. The test may report `success' or `failure'.
If the test reports `success', all the clusters are correctly identified and a certificate is produced.

Our clustering test applies to sum-of-norms clustering with both unit and multiplicative weights.
Given positive parameters $r_1,\ldots,r_n$,
the formulation of multiplicative weights is as follows:
\begin{equation}
    \min_{\x_1,\ldots,\x_n\in\R^d} \frac{1}{2} \sum_{i =1}^n r_i \|\x_i - \a_i\|^2 + \lambda \sum_{1 \le i<j \le n} r_i r_j \|\x_i - \x_j\|\quad \text{where } r_i \ge 0 \; \forall i = 1. \dots, n.
\label{eq:multweight}
\end{equation}
The
utility of multiplicative weights is illustrated in Section~\ref{sec:exper}.
If $r_1,\ldots,r_n$ are integral, then multiplicative weighting is equivalent to
repeating data points.

The general case of weighted clustering allows an arbitrary nonnegative
weights to appear before each term in
\eqref{eq:son-clustering}.
Since unit weights are special case of multiplicative weights, we present all theorems, tests and certificates for the multiplicative weights in the rest of the paper.

Our paper is structured as follows. The test for the unit weights \eqref{eq:son-clustering}, multiplicative weights \eqref{eq:multweight} and the proof of correctness are stated in Section \ref{sec:test}. The proof heavily relies on the sufficient condition for clustering, which is presented in Section \ref{sec:suff_cond}. The test requires the knowledge of a primal and dual feasible solution for the conic formulation of sum-of-norms clustering, which can be constructed from the output of any primal-dual algorithm. The conic formulation and algorithms are stated in Section \ref{sec:feasibility_CS}. If a primal-dual path following algorithm is used,
the test is guaranteed to report `success'
after a finite number of iterations except the test may never report `success' when $\lambda$ is at a fusion value. We include the definition of fusion values in Section \ref{sec:related_work}.
These results are shown in Section \ref{sec:guarantee}. The proof of the theoretical guarantee is a result of the properties of the central path, which are stated in Section \ref{sec:central_path}. In Section \ref{sec:exper}, we present a few computational experiments to verify our test in practice.

\section{Related work}
\label{sec:related_work}
In this paper, we only consider the $l_2$ norm. Nonetheless, the reader should be aware that many other norms such as $l_1, l_\infty$, or the general $l_p$ norms are also extensively studied in the literature of sum-of-norms clustering. Furthermore, the tuning parameter $\lambda$ controls the number of clusters indirectly. When $\lambda=0$, each point is assigned to a cluster of its own. When $\lambda$ is sufficiently large, all points are assigned to the same cluster.
This may lead one to conjecture that
as $\lambda$ increases, clusters may fuse but never break apart, i.e., that the following property holds.

\begin{agglomeration}
\label{def:agglomeration}
  The trajectory $\x^*(\lambda)=[\x_1^*(\lambda),\ldots,\x_n^*(\lambda)]\in\R^{nd}$,
  $\lambda\in[0,\infty)$  of optimizers to
    a weighted formulation of \eqref{eq:multweight} has
    the {\em agglomeration property} if
    whenever $i,j\in n$ and $\lambda\in [0,\infty)$ satisfy
    $\x_i^*(\lambda)=\x_j^*(\lambda)$,
    then $\x_i^*(\lambda')=\x_j^*(\lambda')$
    for all $\lambda'\ge \lambda$.
\end{agglomeration}

Let \textit{fusion values} denote the values of $\lambda$ at which clusters fuse to form a larger cluster. If the agglomeration property holds, there are at most $n$ fusion values. However,  the agglomeration property does not necessarily hold for arbitrary weights. Hocking et al.\ \cite{hocking} observe the occurrence of splits in the $l_2$-norm clusterpath. They conjecture that sum-of-norms clustering is agglomerative when the weights are unit or exponentially decaying (see below). Chiquet et al.\ \cite{chiquet} prove the conjecture for the family of multiplicative weights,
of which unit weights are special case.

\begin{theorem}[Chiquet et al\ \cite{chiquet}]\label{thm:agglomeration}
Sum of norms clustering with multiplicative weights \eqref{eq:multweight} admits the agglomeration property.
\end{theorem}

Exponentially decaying weights have the following formulation.
\begin{equation}
    \min_{\x_1,\ldots,\x_n\in\R^d}\frac{1}{2}\sum_{i=1}^n\norm{\x_i-\a_i}^2+
\lambda\sum_{1 \le i<j \le n}\exp(-\phi\norm{\a_i-\a_j}^2)\norm{\x_i-\x_j}\quad \text{where } \phi \ge 0,
\label{eq:expweight}
\end{equation}
The family of exponentially decaying weights demonstrates promising performance in both simulated and real datasets \cite{ChiLange}.
However, the agglomeration property for \eqref{eq:expweight} is yet to be established. Chi and Steinerberger\ \cite{ChiSteinerberger} partially address the conjecture by proving the existence of some decaying weights which inherit the agglomeration property. They further argue that exponentially decaying weights share certain features of the decaying weights that guarantee agglomeration.

Many algorithms, both primal-only and primal-dual methods, have been proposed to solve \eqref{eq:son-clustering}. Primal-only algorithms include subgradient descent by Hocking et al.\ \cite{hocking} and stochastic incremental algorithm by Panahi et al.\ \cite{Panahi}. Primal-dual algorithms are also widely considered such as ADMM and ADA by Chi and Lange\ \cite{ChiLange}, and semismooth Newton-CG augmented Lagrangian method by Yuan et al.\ \cite{dsun1}. These algorithms solve both \eqref{eq:son-clustering} and its dual problem 
(refer to Section~\ref{sec:feasibility_CS}).
All these iterative algorithms yield only approximate solutions, even though exact knowledge of the optimizer is demanded to determine the clusters.

To identify the correct clusters from an approximate solution, authors in practice propose two methods. The first clustering method applies to all algorithms and makes use of an artificial tolerance, $\epsilon > 0$. If the approximate solution satisfies $\norm{\x_i - \x_{i'}} \le \epsilon$, $i, i'$ are assigned to the same cluster. Otherwise, $i, i'$ are assigned to different clusters. Hence, the value of artificial tolerance is critical.
The other method is specific for primal-dual algorithms such as ADMM and AMA
\cite{ChiLange} and is described at the beginning of
Section~\ref{sec:test}.

The first method is not robust in the following sense.
Since the relation $\norm{\x_i - \x_{i'}} \le \epsilon$ is not transitive, it is not clear how the test would cluster points $i,j,k$ if $\norm{\x_i-\x_j} \le \epsilon, \norm{\x_j-\x_k} \le \epsilon$, and $\norm{\x_i-\x_k} > \epsilon$.
Neither method is associated with an accuracy certificate.
The clusters obtained by these tests could deviate from the clusters corresponding to the optimizer of \eqref{eq:son-clustering}. The inaccuracy may lead to the failure of known properties of sum-of-norms clustering such as the recovery of a mixture of Gaussians and the agglomeration property. It has been established that for the appropriate choice of $\lambda$, \eqref{eq:son-clustering} exactly recovers a mixture of Gaussians due to Panahi et al.\ \cite{Panahi}, Sun et al.\ \cite{dsun1}, and Jiang, Vavasis and Zhai \cite{jiangvavasiszhai}. However, it is unknown if the recovery result still holds when the approximate test is applied. Moreover, Chiquet et al\ \cite{chiquet} prove the agglomeration conjecture (Theorem \ref{thm:agglomeration}) with some techniques which may not be applicable when the approximate test is implemented. Thus the agglomeration property may no longer hold.

To resolve the issue of inaccuracy, Hocking et al.\ \cite{hocking} develop a two-step method based on the first clustering test as described above. The first step is detecting potential fusions using the first clustering test. The artificial tolerance is chosen to be some fraction of the minimum distance between two data points, $\min_{1 \le i < i' \le n}\norm{\a_i - \a_{i'}}$. The second step is verifying potential fusions by checking if the detected fusions improve the objective value. Friedman et al.\ \cite{friedman2007} present a similar approach to detect fusions for a fused-lasso problem with coordinate descent algorithms. The algorithm includes a descent cycle, a fusion cycle and a smoothing cycle. The descent cycle employs coordinate descent to solve a fused-lasso problem. When the coordinate descent gets stuck, the algorithm enters the fusion cycle. The fusion cycle merges any adjacent pairs if the fusion of the pair decreases the objective value. However, it only examines the potential fusions of pairs, but it does not consider the fusions of three points or more. When the fusion cycle fails to merge any adjacent pairs, there may still exist a fusion of three points or more that improves the objective value. To resolve the issue, Friedman et al.\ \cite{friedman2007} introduce a smoothing cycle. The smoothing cycle varies some parameters in the fused lasso problem, which allows fusions of more than two in the long run. Both methods by Friedman et al.\ \cite{friedman2007} and Hocking et al.\ \cite{hocking} guarantee a correct solution. Unfortunately, they are both very slow as they investigate all possible fusion events.
Furthermore, none of the above-mentioned results produces a checkable certificate.

\section{Sufficient conditions on clustering}
\label{sec:suff_cond}
Let $C \subseteq \{1,2,...,n\}$ denote a subset. To draw meaningful conclusions about $C$, we use the sufficient condition in this section to develop our test. Theorem \ref{thm:clust_suff} is due to Chiquet et al.\ \cite{chiquet}, which is a sufficient condition for clustering. The reader may refer to the work by Jiang, Vavasis and Zhai\ \cite{jiangvavasiszhai} for an exposition of Theorem \ref{thm:clust_suff}.

Let $\x^*$ denote the optimal solution to \eqref{eq:multweight}, and let $\x$ denote the output of some primal-dual algorithm which solves \eqref{eq:multweight}. The optimizer $\x^*$ must satisfy the following condition:
\[
  \x_i^*-\a_i-\sum_{j = i+1}^n r_j\bdelta_{ij}^* + \sum_{j = 1}^{i-1} r_j\bdelta_{ji}^* =\bz \qquad\forall i=1,\ldots,n,
\]
where $\bdelta_{ij}^*$
is the subgradient of the Euclidean norm $-\lambda \norm{\x_i - \x_j}$ with respect to $\x_i$
satisfying
$$\bdelta_{ij}^* = \left\{
\begin{array}{ll}
  -\lambda\frac{\x_i^*-\x_j^*}{\norm{\x_i^*-\x_j^*}}, & \mbox{for $\x_i^*\ne \x_j^*$}, \\
  \mbox{arbitrary point in $B(\bz,\lambda)$}, & \mbox{for $\x_i^*=\x_j^*$},
\end{array}
\right.
$$
for all $1 \le i < j \le n$.
We have adopted the following subscripting notation, which
will continue for the remainder of the paper.
Suppose there exists a
sequence of $n(n-1)/2$ vectors
$\{\bdelta_{ij}\in\R^d: 1\le i<j\le n\}$.
Then
\begin{equation}
  \bdelta_{\abrack{ij}} :=\left\{
  \begin{array}{ll}
    \bdelta_{ij}, & i < j, \\
    -\bdelta_{ji}, & i > j, \\
    \bz, & i = j.
  \end{array}
  \right.
\label{eq:abrack_notation}
\end{equation}
This notation renders the first-order conditions more
compact by avoiding separate summations for the
$i<j$ and $i>j$ cases
as shown below:
\begin{equation}
  \x_i^*-\a_i-\sum_{j=1}^n r_j\bdelta_{\langle ij \rangle}^*=\bz \qquad\forall i=1,\ldots,n,
\label{eq:KKTcond}
\end{equation}

We use $B(\c,\rho)$ to denote a closed Euclidean ball centered at $\c$ of radius $\rho$. The condition described above is the key to the proof of Theorem \ref{thm:clust_suff}.


\begin{theorem}
Suppose there exist
$\q_{ij}^*$ for all $i, j \in C, i < j$
solving the following system \eqref{eq:zstar1}. Then there exists some $\hat\x \in \R^d$ such that the minimizer $\x^*$ of \eqref{eq:multweight} satisfies $\x_i^*=\hat\x$ for $i\in C$, hence $C$ is a cluster or part of a larger cluster.
\begin{equation}
\begin{aligned}
\a_i-\frac{1}{r'}\sum_{l\in C}r_l\a_l&=
\sum_{j\in C}r_j \q_{\abrack{ij}}^*, \quad \forall i\in C,\\
\norm{\q_{ij}^*} &\leq \lambda, \quad \forall i,j\in C, i < j 
\end{aligned}
\label{eq:zstar1}
\end{equation}
\label{thm:clust_suff}
Here, $r'=\sum_{i\in C}r_i$
\end{theorem}


\section{Feasibility and complementary slackness}
\label{sec:feasibility_CS}
In this section, we consider a second-order cone (SOCP) formulation of \eqref{eq:multweight}. Both feasibility and complementary slackness are stated. A second-order cone program can be directly solved by a feasible interior-point method. For infeasible algorithms such as the ADMM proposed by Chi and Lange \cite{ChiLange}, we construct a feasible solution for the SOCP from the outputs of such algorithms.

We first present the equivalent SOCP formulation to \eqref{eq:multweight}, which will be derived in this section,
\begin{subequations}
\label{eq:mwsocp_primal}
\begin{align}
    \underset{\x, \y, \z, s, u, t}{\text{min}}
        & \quad f(\x, \y, \z, s, u, t) = \sum_{i=1}^n r_i s_i + \lambda \sum_{1 \le i < j \le n} r_i r_j t_{ij}\label{eq:mwp_obj}\\
    \text{s.t}
        & \quad r_ir_j(\x_i - \x_j - \y_{ij}) = \bz\;, \quad \forall 1 \le i<j \le n\;, \label{eq:mwp_constr1}\\
        & \quad r_i(\x_i - \z_i - \a_i)=\bz\;, \quad \forall i = 1,\ldots,n\;, \label{eq:mwp_constr2}\\
        & \quad r_i(s_i - u_i - 1)=0 \;, \quad \forall i = 1,\ldots,n\;, \label{eq:mwp_constr3}\\
        & \quad t_{ij} \ge \norm{\y_{ij}} \;, \quad \forall 1 \le i<j \le n \;, \label{eq:mwp_constr4}\\
        & \quad s_i \ge \norm{\begin{pmatrix}
    \z_i\\
    u_i
  \end{pmatrix}}\;, \quad\forall i=1,\ldots,n\;. \label{eq:mwp_constr5}
\end{align}
\end{subequations}
The scaling factors $r_ir_j$ in
\eqref{eq:mwp_constr1} and $r_i$ in
\eqref{eq:mwp_constr2}--\eqref{eq:mwp_constr3} do not alter the feasible region but
cause a rescaling of the dual variables, which simplifies later notation.
The SOCP formulation of the dual problem has variables $\bbeta_i\in\R^d,\gamma_i\in\R$
for $i=1,\ldots,n$ and $\bdelta_{ij}\in\R^d$ for $1\le i<j\le n$ and
is as follows.
\begin{subequations}
\label{eq:mwsocp_dual}
\begin{align}
    \underset{\bdelta, \bbeta, \gamma}{\text{max}}
        & \quad h(\bdelta, \bbeta, \gamma) = \sum_{i = 1}^n r_i \a_i^T \bbeta_i + \sum_{i = 1}^n r_i \gamma_i \label{eq:mwd_obj} \\
    \text{s.t}
    & \quad \sum_{j=1}^n r_j \bdelta_{\abrack{ij}} + \bbeta_i=\bz\;, \quad \forall i=1,\ldots,n
    \;, \label{eq:mwd_constr1} \\
        & \quad \lambda \ge \norm{\bdelta_{ij}}\;, \quad \forall 1 \le i<j \le n\;, \label{eq:mwd_constr2}\\
        & \quad 1-\gamma_i \ge \norm{\begin{pmatrix}
        \bbeta_i\\
        \gamma_i
      \end{pmatrix}} \;, \quad \forall i = 1,\ldots,n\;. \label{eq:mwd_constr3}
\end{align}
\end{subequations}
Both primal and dual problems are feasible, and Slater condition holds for both of them. Consider the following primal and dual feasible solution:
\begin{displaymath}
\x_i = \a_i,\; \z_i = \bz, \; s_i = 1, \; u_i = 0, \; \bbeta_i = \bz, \; \gamma_i = 0 \qquad \forall i =1,\ldots,n;
\end{displaymath}
\begin{displaymath}
\y_{ij} = \a_i - \a_j, \; t_{ij} = \norm{\a_i - \a_j} + 1, \; \bdelta_{ij} = \bz \qquad \forall 1\le i<j \le n;
\end{displaymath}
which is a also primal and dual Slater point. Hence, strong duality holds since the problem is formulated as convex optimization.

We derive the SOCP \eqref{eq:mwsocp_primal} as follows. Introduce auxiliary variables $\y_{ij}$ and $\z_i$ and constraints \eqref{eq:mwp_constr1} and \eqref{eq:mwp_constr2}.
Define variables $t_{ij}$ and constraint \eqref{eq:mwp_constr4}.
Introduce variables $s_i$ and $u_i$ satisfying \eqref{eq:mwp_constr3} and
\[
s_i \ge \frac{\norm{\x_i - \a_i}}{2} + \frac{1}{2}, \quad \forall i= 1,\dots, n.
\]
Multiply the constraint in the previous line by 2, and add $s_i^2 - 2 s_i$ to both sides. Simplify the inequality, and substitute $u_i$ to obtain constraint \eqref{eq:mwp_constr5}.

The objective function has the following upper bound using auxiliary variables and new constraints:
\begin{equation}
    \begin{aligned}
    f'(\x) &= \frac{1}{2}\sum_{i=1}^n r_i\norm{\x_i-\a_i}^2 +\lambda\sum_{1 \le i<j \le n}r_ir_j\norm{\x_i-\x_j}\\
    &= \frac{1}{2}\sum_{i=1}^n r_i\|\x_i - \a_i\|^2 + \lambda \sum_{1 \le i < j \le n} r_ir_j\|\y_{ij}\|\\
    &\le \sum_{i=1}^n r_is_i - \frac{1}{2}\sum_{i=1}^nr_i + \lambda \sum_{1 \le i < j \le n} r_ir_jt_{ij} \\
    &= f(\x, \y, \z, s, u, t) - \frac{1}{2}\sum_{i=1}^nr_i \label{eq:ff'}.\\
    \end{aligned}
\end{equation}
Notice that \eqref{eq:multweight} is a minimization problem. For every feasible solution $\x$ to \eqref{eq:multweight}, we can construct a feasible solution $(\x', \y', \z', s', u', t')$ to \eqref{eq:mwsocp_primal} such that $\x = \x'$ and the upper bound \eqref{eq:ff'} is achieved. Hence, we can replace the objective function $f'(\x)$
with the linear function $f(\x, \y, \z, s, u, t)$ as shown in \eqref{eq:mwp_obj}. The original problem \eqref{eq:multweight} and the SOCP \eqref{eq:mwsocp_primal} are indeed equivalent. Since we omit the constant term $\frac{1}{2}\sum_{i=1}^nr_i$ in the objective function of the SOCP, the objective values of \eqref{eq:multweight} at $\x$ and \eqref{eq:mwsocp_primal} at the corresponding solution $(\x', \y', \z', s', u', t')$ differ by this constant.


For the clustering test, we require a primal-dual feasible pair for
\eqref{eq:mwsocp_primal}--\eqref{eq:mwsocp_dual}.  Not all interior-point
methods maintain feasibility.  Furthermore, other primal-dual algorithms
such as the Chi-Lange ADMM \cite{ChiLange} do not compute a full set
of SOCP variables.  We assume that the iterative
algorithm is able to provide at least
an approximately optimal primal solution $\x$ and an
approximately optimal, approximately feasible dual vector $\bdelta$.
This $\bdelta$ appears in \eqref{eq:mwsocp_primal} and also corresponds to the Lagrange
multipliers in the Lagrangian dual that would be produced by ADMM.
From such a pair $(\x,\bdelta)$, it
is possible to round $\bdelta$ to feasibility and produce a full list of the
variables of both \eqref{eq:mwsocp_primal} and \eqref{eq:mwsocp_dual}.  The
derivation of the Lagrangian dual and the formulas for obtaining all SOCP
variables are provided in Section~\ref{sec:SM_A}.

Let $(\x, \y, \z, s, u, t, \bdelta, \bbeta, \gamma)$ be a primal and dual feasible solution for the SOCP formulation of sum-of-norms clustering.
Let us define
residuals of the complementarity conditions
$\bepsilon^{ij} = \begin{pmatrix}
\epsilon_1^{ij}\\
\bepsilon_2^{ij}
\end{pmatrix}$ for all $1 \le i < j \le n$ and
$\bsigma^{i} = \begin{pmatrix}
\sigma_1^{i}\\
\bsigma_2^{i}\\
\sigma_3^{i}\\
\end{pmatrix}$ for all $i= 1, \dots, n$ as follows:

%
\begin{align}
t_{ij} \lambda + \y_{ij}^T \bdelta_{ij} &= \epsilon_1^{ij}, \quad \forall 1\le i<j \le n, \label{eq:mwcs_a1}\\
t_{ij} \bdelta_{ij} + \lambda \y_{ij} &= \bepsilon_2^{ij}, \quad \forall 1\le i<j \le n, \label{eq:mwcs_a2}\\
s_{i} (1 - \gamma_i) + \z_{i}^T \bbeta_{i} + u_i \gamma_i &= \sigma_1^{i}, \quad \forall i = 1, \dots, n\label{eq:mwcs_b1},\\
s_{i} \bbeta_{i} + (1 - \gamma_{i}) \z_{i} &= \bsigma_2^i, \quad \forall i = 1, \dots, n \label{eq:mwcs_b2},\\
s_{i} \gamma_i + (1 - \gamma_{i}) u_i &= \sigma_3^i, \quad \forall i = 1, \dots, n \label{eq:mwcs_b3}.
\end{align}

Define $\mu:= f'(\x) - h'(\bdelta)$ to be the duality gap at the feasible solution. We have $\sigma_1^i, \epsilon_1^{ij}$ upper bounded by $\mu$, and $\norm{\bepsilon^{ij}_2}$ and $\norm{\begin{pmatrix}
\bsigma_2^i\\
\sigma_3^i
\end{pmatrix}}$ upper bounded by $O(\sqrt{\mu})$. Specifically, there hold

%
\begin{equation}
    \norm{\bepsilon_2^{ij}} \le \sqrt{\frac{1}{r_i r_j} \left (\sum_{l=1}^n r_l \norm{\bar \a - \a_l}^2 \mu+ 2\mu^2 \right)}.
\label{eq:mw_bound_epsilon2}
\end{equation}
for all $1 \le i < j \le n$ and
\begin{equation}
\norm{\begin{pmatrix}
\bsigma_2^i\\
\sigma_3^i
\end{pmatrix}} \\
\le \sqrt{\left (\frac{1}{r_i}\sum_{l=1}^n r_l \norm{\bar \a - \a_l}^2 + \frac{2 \mu}{r_i} + 1 \right) \cdot \left (\frac{1}{2} + \frac{1}{r_i} \left (\sum_{l=1}^n r_l (r' - r_i) \lambda \norm{\a_l} + \mu \right ) \right ) \cdot \mu}.
\label{eq:mw_bound_sigma23}
\end{equation}
for all $i = 1, ..., n$. The derivation is also included in Section \ref{sec:SM_B_mw}.

\section{Clustering test}
\label{sec:test}
In this section, we present a clustering test that applies to sum-of-norms clustering with the family of multiplicative weights 
stated in \eqref{eq:multweight}.
Given a primal and dual feasible solution $(\x, \y, \z, s, u, t, \bdelta, \bbeta, \gamma)$ with a duality gap $\mu$,
we first find candidate clusters using either method described in Section \ref{sec:intro}. To employ the first method, we select an index $i$ from $[n]$ arbitrarily. Construct a ball of radius $\mu^{0.75}$ about $\x_i$. Create a candidate cluster with all indices $k$ such that $\x_k$ is located in the ball about $\x_i$ (i.e. $\{k:\norm{\x_i-\x_k}\le\mu^{3/4}\}$). Now find an index $j$ that is not in any candidate cluster and construct a ball about $\x_j$. Repeat until all data points are used up. To implement the second method, we compute
$\v_{ij} = \mathrm{prox}_{\lambda/\nu \norm{\cdot}} (\x_i - \x_j - \nu^{-1} \bdelta_{ij})$ for all $i, j$
for the current iterate,
where $\nu$ is a chosen parameter (it corresponds to the augmented Lagrangian parameter in ADMM). We assign
candidate
clusters based on the graph induced by $\v_{ij}$'s (see Chi and Lange \cite{ChiLange}). In the graph, each node corresponds to one data point. An edge connects nodes $i,j$ if and only if $\v_{ij} = \bz$. Apply breadth-first search to identify the connected components. Then each connected component is a cluster.

If the output of the primal-dual algorithm is not feasible for the SOCP, we construct a feasible solution as described in the previous section. With the feasible solution, we define
\begin{equation}
  \bomega_i := \frac{\sigma_3^i}{s_i} \z_i + \frac{1}{s_i} \bsigma_2^i,
  \quad \forall i.
  \label{eq:subgrad_omega_def}
\end{equation}

For any candidate cluster $C$, we compute
\begin{equation}
  \q_{ij}:= -\bdelta_{ij}
  + \frac{1}{r'} \cdot (\x_i - \x_j - \bomega_i + \bomega_j) - \frac{1}{r'} \sum_{k \notin C} r_k (\bdelta_{\abrack{ik}} - \bdelta_{\abrack{jk}})
  \label{eq:q_def}
\end{equation}
for all $i, j \in C, i < j$.
We use $r'=\sum_{i\in C}r_i$ to denote the sum of weights for a cluster $C$.
For any distinct pair of candidate clusters $C_k \cup C_{k'}$, define $\bar \x:= \sum_{l\in C_k \cup C_{k'}}\frac{r_l}{\sum_{l'\in C_k \cup C_{k'}} r_{l'}} \x_l$
to be the
weighted
centroid of $C_k \cup C_{k'}$. Compute
 $D_{k,k'}:=
\sum_{l \in C_k \cup C_{k'}} r_l \norm{\x_l - \bar \x}^2$.
Check if the following two conditions hold:

\textbf{\textit{CGR subgradient condition}:} All CGR subgradients
$\q_{ij}$
satisfy the CGR inequality
$\norm{\q_{ij}} \le \lambda$.

\textbf{\textit{Separation condition}:} All distinct pairs of candidate clusters $C_k, C_{k'}$ satisfy $D_{k,k'} > 2\mu$.
If both conditions hold for all candidate clusters, then the test terminates and reports `success'. Each candidate cluster is a real cluster given by the optimal solution, thus all clusters are correctly identified. The
$\q_{ij}$'s
serve as certificates. If either condition fails for any candidate cluster, the test reports `failure'. One has to run more iterations of the algorithm to decrease the duality gap $\mu$. Repeat the process until the test reports `success'. Note that this test is algorithm-independent, but it does require the algorithm to be of primal-dual type.

The CGR subgradient test is validated by Theorem \ref{thm:clust_suff} as presented in Section \ref{sec:suff_cond}. The CGR subgradients condition certifies that each cluster we identify is indeed a cluster or part of a larger cluster. This is presented in Section \ref{sec:cgr_subgradient}. The separation condition certifies that there is no super-cluster with more than one cluster we identify by Theorem \ref{thm:noclust_suff}, as shown in Section \ref{sec:CGR_duality_gap}. Therefore, we determine all clusters correctly when the test succeeds.






\subsection{CGR subgradients and clustering corollary}
\label{sec:cgr_subgradient}
Let $C \subseteq [n]$ denote a subset of points. Let $m:= |C|$ denote the cardinality of $C$.

\begin{lemma}
For all $i, j \in C, i < j$,
$\q_{ij}$
as defined in \eqref{eq:q_def} satisfies
\begin{align}
  \a_i-\frac{1}{r'} \sum_{l\in C}r_l \a_l &= \sum_{j\in C} r_j \q_{\abrack{ij}}, \quad \forall i \in C
\end{align}
where $r' = \sum_{i \in C} r_i$.
\label{cluster_lemma_subgradient}
\end{lemma}

The proof of this lemma is deferred to the supplementary material \ref{sec:SM_2}.


\begin{corollary}
If
$\norm{\q_{ij}} \le \lambda$
holds for all $i,j \in C, i < j$, where $C$ is a candidate cluster, then $C$ is a cluster or part of a larger cluster.
\end{corollary}
The proof of the corollary follows trivially by Theorem \ref{thm:clust_suff}.

\subsection{Duality gap and distinct clustering theorem}
\label{sec:CGR_duality_gap}

We prove of the separation condition as follows.

\begin{lemma}
  If a distinct pair of candidate clusters $C_k, C_{k'}$ satisfy $D_{k,k'} > 2\mu$, then there does not exist a super-cluster that
  contains both $C_k \cup C_{k'}$.
\label{thm:noclust_suff}
\end{lemma}

\begin{proof}
  We prove the contrapositive.
  Assume $C_k, C_{k'}$ are two
  disjoint
  subsets of a larger cluster, which implies that $\x^*_i = \x^*_j$ for all $i, j \in C_k \cup C_{k'}$. First, we rewrite \eqref{eq:multweight} in a matrix form as follows:
\[
    \min_{\x\in\R^{nd}} f(\x) := \frac{1}{2} (\x - \a) R (\x-\a) + \lambda q(\x),
\]
where $R = \mathop{\rm Diag}([r_1 \mathbf 1_d, \cdots, r_n \mathbf 1_d])$ denotes the weight matrix and $q(\x) := \sum_{1 \le i<j \le n} r_i r_j \|\x_i - \x_j\|$ denotes the regularization term.
By optimality, it holds that
\[
R (\x^* - \a) + \lambda \w^* =
\bz,
\]
where $\w^*$ is a subgradient of $q$ at $\x^*$.
Notice that $q$ is a convex function by definition. By the subgradient inequality,
\[
q(\x^*+\h) \ge q(\x^*) + \w^{*T}\h.
\]
for all $\h \in \mathbb R^{nd}$
Hence.
\begin{align*}
    f(\x^* + \h) &= \frac{1}{2} (\x^* - \a + \h)^T R (\x^* - \a + \h) + \lambda q(\x^* + \h)\\
    &\ge \frac{1}{2} (\x^* - \a)^{T} R (\x^* - \a) + (\x^* - \a)^{T} R \h + \frac{1}{2} \h^T R \h +
    \lambda q(\x^*) + \lambda \w^{*T} \h\\
    &= f(\x^*) + (\x^* - \a)^{T} R \h + \frac{1}{2} \h^T R \h + \lambda \w^{*T} \h\\
    &= f(\x^*) + ((\x^* - \a)^{T} R + \lambda \w^{*T}) \h + \frac{1}{2} \h^T R \h \\
    &= f(\x^*) + \frac{1}{2} \h^T R \h.\\
\end{align*}
By weak duality, we have
\[
\mu \ge f(\x^* + \h) - f(\x^*) \ge \frac{1}{2} \h^T R \h.
\]
Set $\h:= \x - \x^*$ to be the difference of the current iterate $\x$ and the optimal solution $\x^*$. Since $R$ is positive definite, we have
\[
\frac{1}{2} \h^T R \h \ge \frac{1}{2} \sum_{i \in C_k \cup C_{k'}} r_i \norm{\x_i - \x_i^*}^2.
\]
Under our assumption, $\x_i^*$'s share the same value. The RHS achieves it minimum at $\x_i^*=\bar \x:= \sum_{l\in C_k \cup C_{k'}}\frac{r_l}{\sum_{l'\in C_k \cup C_{k'}} r_{l'}} \x_l$. Substituting the result to the previous two inequalities yields
\[
2\mu \ge \sum_{i \in C_k \cup C_{k'}} r_i \norm{\x_i - \bar \x}^2 =: D_{k, k'}.
\]
thus establishing the contrapositive of the lemma.
\end{proof}

\section{Properties of the central path}
\label{sec:central_path}
In this section, we explore the properties of the central path for a primal-dual path following algorithm. These properties play a fundamental role in the proof of our main theorem in Section \ref{sec:guarantee}. In the main theorem, we state that if a primal-dual path following algorithm is used, our clustering test will eventually succeed after a finite number of iterations when $\lambda$ is not at any fusion value. The proof of the ultimate success relies on the linear convergence to the optimal primal-dual pair, which will be shown to be satisfied in the remainder of this section.

Even though there are very few theorems about the central path of second-order-cone programming in literature, there are established theorems from semidefinite programming (SDP). SDP specializes to SOCP. With some standard techniques, we can easily rewrite our SOCP problem as SDP and apply the primal-dual path following algorithm to solve the new SDP. The following theorem states that the $\mu'$-centered iterates converge to the analytic center superlinearly.

\begin{theorem}[Luo et al \ \cite{Luo}]
Assume the semidefinite program has a strictly complementary solution and the iterates of the algorithm converge tangentially to the central path. Let $(X(\mu'), Z(\mu'))$ denote a $\mu'$-centered primal-dual pair. Let $(X^a, Z^a)$ denote the analytic centers of the primal and dual optimal sets. Let $\mu' \in (0,1)$ be the central path parameter. There holds
\[
\norm{X(\mu')-X^a} = O(\mu'), \quad \norm{Z(\mu')-Z^a} = O(\mu').
\]
\label{thm:Luo}
\end{theorem}

Assume a primal-dual path following algorithm satisfying the assumptions of Luo et al., and it is applied to solve the SOCP as SDP. To employ Theorem \ref{thm:Luo}, we show that our SOCP has a strictly complementary optimizer provided $\lambda$ is not a fusion value $\lambda^*$.
The failure at fusion values is not surprising since any arbitrarily small negative perturbation $\lambda^* + \epsilon$ yields a different clustering. In other words, complete cluster identification for these fusion values is ill-posed. Thus it is unreasonable to expect an algorithm that satisfies a guarantee for such a problem. There are at most $n$ fusion values as a result of Theorem \ref{thm:agglomeration}.

It is worth remarking that Theorem \ref{thm:Luo} does not directly apply to a primal-dual SOCP interior-point method. SOCP is a special case of SDP, yet the central path of SOCP is not just a simple projection of the SDP central path. The reason is that the log-barrier function for SOCP is not a specialization of the log-barrier function for SDP. Let $\x$ denote a primal feasible solution for an SOCP where $\x = \begin{pmatrix}
x_0\\
\bar \x
\end{pmatrix} \in \mathbb R^{d+1}$ satisfies $x_0 \ge \norm{\bar \x}$. Then the log-barrier function inherited from SDP reformulation would be
\[
\phi_{SDP}(\x) = -\ln(x_0^2 - \norm{\bar \x}^2) - (d-1) \ln x_0,
\]
while the log-barrier function inherited from the original SOCP would be
\[
\phi_{SOCP}(\x) = -\ln(x_0^2 - \norm{\bar \x}^2).
\]
The removal of the second term accelerates the convergence.

We suspect that an SOCP interior-point method should also satisfy a bound analogous to Theorem \ref{thm:Luo}, but we are not aware of
a
proof in the literature.

\subsection{Strict complementarity}
\label{sec:strict_complementarity}
By specializing the definition of strict complementarity in SDP to SOCP\ \cite{goldfarb}, a primal and dual optimal solution satisfies strict complementarity if and only if
\begin{align}
  t_{ij} + \lambda &> \|\y_{ij} + \bdelta_{ij}\|, \quad \forall 1 \le i < j \le n, \label{eq:strict_cs_a}\\
    s_{i} + 1 - \gamma_i &> \norm{
    \begin{pmatrix}
    \z_{i} + \bbeta_i\\
    u_i + \gamma_i
    \end{pmatrix}}, \quad \forall i = 1, ..., n\label{eq:strict_cs_b}
\end{align}

The following theorem is a sufficient condition for strict complementarity of \eqref{eq:mwsocp_primal} and \eqref{eq:mwsocp_dual}.

\begin{theorem}
If $\lambda > 0$ is a parameter value at which fusion does not occur, then there exists a strictly complementary primal-dual optimal solution to SOCP \eqref{eq:mwsocp_primal} and \eqref{eq:mwsocp_dual} at $\lambda$.
\label{thm:strict_complementarity_socp}
\end{theorem}

To prove Theorem \ref{thm:strict_complementarity_socp}, we consider a new optimization problem and construct such a strictly complementary primal-dual optimal solution from the new problem. Let $\lambda_1, \lambda_2$ be the two successive fusion values such that $\lambda \in (\lambda_1, \lambda_2)$. Note that it is possible for $\lambda_1 = 0$ or $\lambda_2 = \infty$. Let $(\x', \y', \z', s', u', t', \bdelta', \bbeta', \gamma')$
denote a primal and dual optimal solution at $\lambda_1$. Let $C_1, C_2, ..., C_K$ denote the clusters identified by the optimal solution above. When $\lambda_1 = 0$, there are $n$ clusters, and each cluster is a singleton set. When $\lambda_1$ is the largest fusion value, there is only one cluster containing all $n$ points.

For each $k=1,\ldots,K$,
define $r_k'=\sum_{i\in C_k}r_i$ and $\bar\a_k=\frac{1}{r_k'}\sum_{i\in C_k}r_i\a_i$.
Consider the following optimization problem:
\begin{equation}
    \min_{\x_1, ..., \x_K\in\R^d} \frac{1}{2}\sum_{k=1}^K r_k' \norm{\x_k-\bar\a_k}^2 +\lambda\sum_{1 \le k<k' \le K} r_k'r_{k'}'\norm{\x_k-\x_{k'}}.
    \label{eq:son-clustering_multweight}
\end{equation}

Let $\x$ denote the optimal solution of \eqref{eq:son-clustering_multweight}.

\begin{lemma}
Vector $\x$ satisfies $\x_k \ne \x_{k'}$ for all $k, k' \in [K], k \ne k'$. \label{lemma_pkpk'}
\end{lemma}

\begin{proof}
For the purpose of contradiction, we may assume there exist $\hat k \ne \hat k'$ such that $\x_{\hat k} = \x_{\hat k'}$.

The first step is to construct an optimal solution $\x^*$ to the original problem \eqref{eq:multweight} at $\lambda$ using $\x$. Let $\x^*_i = \x_k$ for $i \in C_k, k \in [K]$. By the first-order optimality condition of \eqref{eq:son-clustering_multweight} at $\x$, there exist
$\bdelta_{kk'} \in -\lambda \partial \|\x_k - \x_{k'}\|$
with respect to $\x_k$ for all
$k < k'$
such that the equality below holds as analogous to \eqref{eq:KKTcond}. Note that
$\bdelta_{kk'} \in -\lambda \partial\|\x_i^* - \x_j^*\|$
is also a subgradient at $\x^*$
with respect to $\x_i$
for all $i \in C_k, j \in C_{k'}, i < j$.
Thus, for any $i\in C_k$,
\begin{equation}
  \begin{aligned}
\bz &= \x_k - \bar \a_k - \sum_{k'=1}^K r'_{k'} \bdelta_{\abrack{kk'}}\\
&=\x_k - \a_i + \a_i - \bar \a_k - \sum_{k' \ne k} r'_{k'} \bdelta_{\abrack{kk'}}.
  \end{aligned}
\label{eq:p_optimality}
\end{equation}
By the feasibility and complementary slackness in
Section~\ref{sec:SM_2},
the dual solution satisfies
\begin{equation}
    \a_i - \bar \a_k = -\sum_{j \in C_k, } r_j \bdelta_{\abrack{ij}}', \quad \forall i \in C_k, k \in [K], \qquad \text{and} \quad \norm{\bdelta_{ij}'} \le \lambda_1, \quad \forall
    i < j.
    \label{eq:son-clustering_chiquet}
\end{equation}
By construction,
for any $i, j \in C_k, i < j$,
there holds $\bdelta_{ij}' \in -\lambda \partial \norm{\x_i^* - \x_j^*}$
with respect to $\x_i$ since
$\x_i^* = \x_j^*$
and $\lambda_1\le \lambda$.
Substitute \eqref{eq:son-clustering_chiquet} to \eqref{eq:p_optimality} to obtain
\begin{align*}
  \bz 
&=  \x_k - \a_i - \sum_{j \in C_k} r_j \bdelta_{\abrack{ij}}' - \sum_{k'=1}^K \sum_{j \in C_{k'}} r_{j} \bdelta_{\abrack{kk'}}
\end{align*}
satisfying \eqref{eq:KKTcond} at $i$. As $i \in C_k, k \in[K]$ are chosen arbitrarily, the equality \eqref{eq:KKTcond} holds for all $i$ hence $\x^*$ is an optimal solution to \eqref{eq:multweight}. Since $\x_{\hat k} = \x_{\hat k'}$,
we have $\x^*_i = \x^*_j$ for all $i, j \in C_{\hat k} \cup C_{\hat k'}$. By the agglomerative properties of the clusterpath, cluster $C_{\hat k}, C_{\hat k'}$ merge at some $\lambda' \in (\lambda_1, \lambda]$, which contradicts our choice of $\lambda_2$. That concludes our proof.
\end{proof}

By Lemma \ref{lemma_pkpk'}, the objective function
of \eqref{eq:son-clustering_multweight}
is differentiable at $\x$. Hence, there holds
\begin{equation}
\x_k - \bar \a_k + \lambda \sum_{k' \ne k} r'_{k'}\cdot \frac{\x_k - \x_{k'}}{\|\x_k - \x_{k'}\|} = \bz,
\quad \forall k \in [K]. \label{eq:son-clustering_multweight_oc}
\end{equation}

Define the following primal-dual solution:
\begin{equation}
    \begin{aligned}
    \x^*_i &= \x_k, \quad \forall i \in C_k, k \in [K]\\
    \y_{ij}^* &= \x^*_i - \x^*_j, \quad \forall 1 \le i < j \le n\\
    \z^*_{i} &= \x^*_i - \a_i, \quad \forall i = 1, \dots, n, \\
    s_i^* &= \frac{1}{2} (1 + \|\z_{i}^*\|^2), \quad \forall i = 1, \dots, n\\
    u_i^* &= \frac{1}{2} (-1 + \|\z_{i}^*\|^2), \quad \forall i = 1, \dots, n\\
    t_{ij}^* &= \|\y_{ij}^*\|, \quad \forall 1 \le i < j \le n\\
    \bdelta_{ij}^* &= \left \{ \begin{array}{ll}
        \bdelta_{ij}', \quad &\mbox{if } i,j \in C_k\\
         \lambda \frac{\x^*_{j} - \x_j^*}{\|\x_j^* - \x_i^*\|}, & \mbox{otherwise}
    \end{array}
    \right., \quad \forall 1 \le i < j \le n\\
    \bbeta_i^* &= - \z^*_i, \quad \forall i = 1, \dots, n\\
    \gamma_i^* &= \frac{1}{2}(1 - \|\bbeta^*_i\|^2), \quad \forall i = 1, \dots, n
    \end{aligned}
    \label{eq:feasible_sol}
\end{equation}

\begin{lemma}
\label{lemma_optimality_socp}
The solution defined by \eqref{eq:feasible_sol} is optimal for SOCP \eqref{eq:mwsocp_primal} and \eqref{eq:mwsocp_dual} at $\lambda$.
\end{lemma}
\begin{lemma}
The solution defined by \eqref{eq:feasible_sol} is strictly complementary.
\label{lemma_strict_complementarity}
\end{lemma}

\begin{proof}
  The strict complementarity is equivalent to \eqref{eq:strict_cs_a} and \eqref{eq:strict_cs_b}, which can be easily checked as shown below\\
\textbf{Verification of \eqref{eq:strict_cs_a}:} Let $1 \le i<j \le n$. If $\y_{ij}^* = \bz$, then there exists some $k \in [K]$ such that $i, j \in C_k$. By definition, $t_{ij}^* = 0$ and $\bdelta_{ij}^* = \bdelta_{ij}$. Notice that $\bdelta_{ij}$ is the optimal dual solution to \eqref{eq:multweight} at $\lambda_1$, then it satisfies $\|\bdelta_{ij}\| \le \lambda_1 < \lambda$ by the definition of $\lambda$.
\[
t_{ij}^* + \lambda = \lambda > \|\bdelta_{ij}\| = \|\bdelta_{ij}^*\| = \|\y_{ij}^* + \bdelta_{ij}^*\|.
\]
If $\y_{ij}^*\ne \bz$, then complementarity immediately implies strict complementarity
since $[\lambda;\bdelta_{ij}^*]\ne \bz$ (since $\lambda >0$).

\textbf{Verification of \eqref{eq:strict_cs_b}:} Let $i \in [n]$. By construction, $\begin{pmatrix}
  s_i\\
  \z_i\\
  u_i
\end{pmatrix}$ and
$\begin{pmatrix}
  1-\gamma_i\\
  \bbeta_i\\
  \gamma_i
\end{pmatrix}$
are on the boundary of the second order cone
and neither is zero.
\end{proof}

The proof of Lemma \ref{lemma_optimality_socp} is deferred to supplementary material \ref{SM_three_lemmas}. With three lemmas presented in this section, there exists a strictly complementary optimal solution (as defined by \eqref{eq:feasible_sol}) to SOCP \eqref{eq:mwsocp_primal} and \eqref{eq:mwsocp_dual}.

\section{Test Guarantee}
\label{sec:guarantee}
In Section \ref{sec:test}, we validated our test theoretically in the sense that if the test succeeds, it is guaranteed that the correct clusters are found. In this section, we show that the test succeeds after a finite number of iterations of a certain interior point method, provided that $\lambda$ is not at any fusion value. Specifically, we prove that the two conditions in our test are guaranteed to hold for a primal-dual path following algorithm satisfying the assumptions of Luo et al.\ \cite{Luo} when the duality gap $\mu$ is sufficiently small.


\begin{theorem}
If $\lambda$ is not a fusion value, then there exists $\mu_0 > 0$ such that both CGR subgradient and separation conditions in the test are satisfied for any duality gap $\mu \le \mu_0$ for a primal-dual path following algorithm satisfying the assumptions of Luo et al. \cite{Luo}.
\label{thm:test_guarantee}
\end{theorem}

Let $(\x, \y, \z, s, u, t, \bdelta, \bbeta, \gamma)$ denote a primal and dual feasible solution. Let $C_1, C_2, ..., C_K$ denote the clusters obtained at optimum. Let $\mu' \in (0,1)$ denote the central path parameter and let $\mu$ denote the duality gap at the feasible solution. By Theorem \ref{thm:Luo}, there hold
\[
\norm{\x(\mu') - \x^a} = O(\mu'), \quad \norm{\bdelta(\mu) - \bdelta^a} = O(\mu')
\]
where $\x(\mu'), \bdelta(\mu')$ are $\mu'$-centered solutions and $\x^a, \bdelta^a$ are the analytic centers of the primal and dual optimal sets respectively. Moreover, since the iterates converge tangentially to the central path, we may assume the size of the central path neighborhood to be as follows
\[
\norm{\x - \x(\mu')} = O(\mu'), \quad \norm{\bdelta - \bdelta(\mu')} = O(\mu').
\]
Luo et al.\ \cite{Luo} validated the assumption above for their interior point algorithm, which is a generalization of the Mizuno-Todd-Ye predictor-corrector method for linear programming. Combine the two sets of equations above and employ the triangle inequality to obtain
\[
\norm{\x - \x^a} = O(\mu'), \quad \norm{\bdelta - \bdelta^a} = O(\mu').
\]
As the duality gap $\mu$ is of a linear order of the central path parameter $\mu'$, the equalities above are rewritten as
\[
\norm{\x - \x^a} = O(\mu), \quad \norm{\bdelta - \bdelta^a} = O(\mu).
\]

Define $p, p' \ge 0$ such that $\norm{\x_i - \x_i^a} \le p \mu$ for all $i$ and $\norm{\bdelta_{ij} - \bdelta^a_{ij}} \le p' \mu$ for all distinct pairs $(i,j)$. Then, for all distinct pairs $(i, j)$ in any cluster $C_k$, there holds $\norm{\x_i - \x_j} \le 2p \mu$. Moreover, define $q > 0$ such that all $\x_i^a$'s in different clusters are at least $q$ apart, which implies that $\x_i$'s in different clusters are separated by a distance of at least $q - 2p \mu$. We may assume the duality gap satisfies $\mu < \frac{q}{2p}$. Notice that this assumption is guaranteed to be true after a finite number of iterations.

\subsection{Bound CGR subgradients}\label{sec:delta}
Let $C:= C_k$ for some $k \in [K]$.
Recall that
$\q_{ij}:= -\bdelta_{ij} + \frac{1}{r'} \cdot (\x_i - \x_j - \bomega_i + \bomega_j) - \frac{1}{r'} \sum_{k \notin C} r_k (\bdelta_{\abrack{ik}} - \bdelta_{\abrack{jk}})$ for $i,j \in C_k, i < j$.
To have
$\norm{\q_{ij}} \le \lambda$ as required in the CGR subgradient condition, we establish an upper bound on the norm of each term in the construction of $\q_{\abrack{ij}}$. The proof for the following upper bounds are attached in Section \ref{SM_C} of the supplementary material. The proof of the following lemmas relies on the key observations: both $\norm{\bepsilon_2^{ij}}$ and $\norm{\begin{pmatrix}
\bsigma_2^i\\
\sigma_3^i
\end{pmatrix}}$ have upper bounds of $O(\sqrt{\mu})$ as shown in \eqref{eq:mw_bound_epsilon2} and \eqref{eq:mw_bound_sigma23}.
\begin{lemma}
For all $i,j \in C, i < j$, the following inequality holds
\[
\norm{\bdelta_{ij}} \le \lambda - r + p' \mu
\]
where $r:= \min_{l \ne l', l, l' \in C_k, k \in [K]} (\lambda - \norm{\bdelta^a_{ll'}}) > 0$.
\label{lemma_bdelta}
\end{lemma}

\begin{lemma}
  For all $i,j \in C$, $k \notin C$
  the following inequality holds
\[
\norm{\bdelta_{\abrack{ik}} - \bdelta_{\abrack{jk}}} \le \frac{4 \lambda p \mu}{q - 2p \mu} +\frac{\left (\sqrt{ \frac{1}{r_i r_k}} + \sqrt{\frac{1}{r_j r_k}}\right ) \cdot \sqrt{\sum_{l=1}^n r_l \norm{\bar \a - \a_l}^2 \mu+ 2\mu^2 }} {q - 2p \mu} + \frac{\mu}{q - 2p \mu}
\]
\label{lemma_g_diff}
\end{lemma}
\begin{lemma}
For all $i\in C$, it holds
\[
\norm{\bomega_i} \le 2 \sqrt{ 2 \cdot \left (\frac{1}{r_i}\sum_{l=1}^n r_l \norm{\bar \a - \a_l}^2 + \frac{2 \mu}{r_i} + 1 \right) \cdot \left (\frac{1}{2} + \frac{1}{r_i} \left (\sum_{l=1}^n r_l (r' - r_i) \lambda \norm{\a_l} + \mu \right) \right ) \cdot \mu}.
\].
\label{lemma_bomega}
\end{lemma}

\begin{lemma}
For all $i, j \in C$ and $i < j$, there holds
\begin{equation}
    \begin{aligned}
    &\norm{\q_{ij}}\\
     \le & \frac{2}{r'} \sqrt{ 2 \cdot \left (\frac{1}{r_i}\sum_{l=1}^n r_l \norm{\bar \a - \a_l}^2 + \frac{2 \mu}{r_i} + 1 \right) \cdot \left (\frac{1}{2} + \frac{1}{r_i} \left (\sum_{l=1}^n r_l (r' - r_i) \lambda \norm{\a_l} + \mu \right) \right ) \cdot \mu} \\
    + & \frac{2}{r'} \sqrt{ 2 \cdot \left (\frac{1}{r_j}\sum_{l=1}^n r_l \norm{\bar \a - \a_l}^2 + \frac{2 \mu}{r_j} + 1 \right) \cdot \left (\frac{1}{2} + \frac{1}{r_j} \left (\sum_{l=1}^n r_l (r' - r_j) \lambda \norm{\a_l} + \mu \right) \right ) \cdot \mu} \\
    +& \frac{1}{r'} \sum_{k \notin C} r_k \left ( \frac{4 \lambda p \mu}{q - 2p \mu} + \frac{2 \sqrt{\sum_{l=1}^n \norm{\bar \a - \a_l}^2 \mu + 2\mu^2}}{q - 2p \mu} + \frac{\mu}{q - 2p \mu} \right) + \lambda - r + p' \mu + \frac{2 p \mu}{r'}
    \end{aligned}
    \label{eq:cgr_subgradient_norm}
\end{equation}
%
\label{lemma_main}
\end{lemma}

\subsection{Proof of Theorem \ref{thm:test_guarantee}}
\label{sec:main_proof}
\begin{proof}
  We rewrite \eqref{eq:cgr_subgradient_norm} with $O(\cdot)$ notation to obtain the following inequality
\[
\norm{\q_{\abrack{ij}}} \le \lambda - r + O(\sqrt{\mu}), \quad \forall i, j \in C_k, k \in [K],
\]
since $C=C_k$ is an arbitrarily cluster. As $r>0$ by Lemma \ref{lemma_bdelta}, there exists $\mu_1 > 0$ such that for all $\mu \le \mu_1$, $\norm{\q_{\abrack{ij}}} \le \lambda$ holds for all $i, j \in C_k, k \in [K]$. Here concludes the proof of the CGR subgradient condition.


Let $C_k, C_{k'}$ denote a pair of distinct clusters. Since $q>0$,
there exist $j \in C_k \cup C_{k'}$ such that $\norm{\x_j - \bar \x}^2 \ge \frac{1}{4}q^2$, where
$\bar \x:= \frac{1}{\sum_{i\in C_k\cup C_{k'}}r_i}\sum_{i \in C_k \cup C_{k'}} r_i\x_i$.
Recalling that $D_{k,k'}=\sum_{i \in C_k \cup C_{k'}}r_i\norm{\x_i - \bar \x}^2$,
we conclude that
$D_{k,k'} \ge \frac{1}{4}r_{\min}q^2$, where $r_{\min}=\min\{r_l:l\in C_k\cup C_{k'}\}$.
Pick $0<\mu_2<\frac{1}{8}q^2r_{\min}$. Then  $D_k> 2 \mu_2$. Since $C_k, C_{k'}$ is an arbitrary pair of distinct clusters, $D_{k,k'} > 2 \mu_2$ is true for all pairs of distinct clusters. Here concludes the separation condition.
Let $\mu_0 = \min \{\mu_1, \mu_2\}$, then both CGR subgradient and separation conditions are satisfied for any $\mu \le \mu_0$.
\end{proof}

\subsection{Complexity}
We omit the complexity analysis of our test for the following reasons. To get complexity bounds for our test, we would have to rederive all the complexity bounds from the work of Luo et al.\ \cite{Luo}.
Luo et al.\ adopt the $O(\cdot)$ notation from the early stage of their proof. The constants in the bounds $O(\mu')$ in Theorem \ref{thm:Luo} are omitted from their paper. However, we present our bounds in both the $O(\cdot)$ form and the exact form with data dependencies. Hence, the bottleneck for a complexity analysis lies in the rederivation of
\cite{Luo}.

\section{Computational experiments}
\label{sec:exper}
In this section, we examine the performance of our clustering test for sum-of-norms clustering with multiplicative weights \eqref{eq:multweight}.
Our experiments indicate that correctly chosen multiplicative weights
increase the recovery power of sum-of-norms clustering, similar to
the increase observed in the literature for exponentially decaying weights
mentioned in Section \ref{sec:related_work}.
However, multiplicative weights have
the advantage over exponential weights
of preserving most of the known strong properties of unit weights including the agglomeration property. Besides the strong theoretical properties, use of multiplicative weights also improves the computational complexity. The ADMM update for the multiplicative weights can be computed using a low-rank matrix update as observed by Chi and Lange \cite{ChiLange}. However, ADMM update of general exponentially decaying weights involves solving a full dimensional dense linear system to find the solution. Yuan et al. \cite{dsun1} observe that the vanilla version of ADMM proposed by \cite{ChiLange} is not scalable. Hence, they design a semismooth Newton-CG to solve their subproblem. For the reasons mentioned above, the method of multiplicative weights is our primary method of interest for experiments.


It is important to stress that we do not have a systematic way to produce
multiplicative weights with high recovery power, and indeed, our constructions
of multiplicative weights in this section assume prior knowledge of the clustering
solution.  We suspect that a systematic way exists, but we postpone consideration
of this question to future work since the issue of improved
recovery power is outside the scope of the
certification question considered herein.


We implement the experiment in which the ADMM solver, the cluster-finding algorithm and our clustering test are applied to \eqref{eq:multweight} on both datasets of two half moons and a mixture of Gaussians.
We remark that we have not established for ADMM a theorem analogous to
Theorem~\ref{thm:test_guarantee}, and therefore, it is possible that ADMM
  may fail to certify a clustering even after an arbitrary number
  of iterations.
We intend to answer the following three questions:
(1) Is the certification successful after a reasonable number
of iterations? and
(2) How closely do the certified clusters match the
generative model's ground-truth clusters?
(3) Does the SON clustering with
multiplicative weights (currently chosen using prior knowledge of the solution)
outperform the SON clustering with unit weights?
Our algorithm is implemented in Julia \cite{Julia2017} as shown below.
\begin{algorithm}
\caption{An ADMM algorithm with our clustering test}
\label{alg:admm_wtest}
\begin{algorithmic}
\STATE{Initialize $(\x, \bdelta)$}
\WHILE{clustering test fails or maximum number of iterations is not reached}
\FOR {$l = 1,2,\ldots,t$}
\STATE {ADMM updates by Chi and Lange \cite{ChiLange}}
\ENDFOR
\STATE{Construct a feasible solution for SOCP by \eqref{eq:feasible_sol} from the current ADMM iterate}
\STATE{Compute the duality gap $\mu$}
\STATE{Run Algorithm \ref{alg:find_clusterA} or \ref{alg:find_clusterB} (the parameter $\nu$ is the same augmented Lagrangian parameter from ADMM) to find clusters $\{R_1, R_2, \ldots, R_{K'}\}$}
\STATE{Compute CGR subgradients from dual variables for $\{R_1, R_2, \ldots, R_{K'}\}$}
\STATE{Check the CGR subgradient condition; Check that the separation condition}
 \STATE{ Mark the clustering test `success' if both conditions pass and mark it `failure' otherwise}
\ENDWHILE
\RETURN candidate clusters $\{R_1, R_2, \ldots, R_{K'}\}$
\end{algorithmic}
\end{algorithm}

\begin{algorithm}
\caption{Find clusters (A)}
\label{alg:find_clusterA}
\begin{algorithmic}
\STATE{Define $C \gets \{1,\ldots, n\}$,  $k \gets 1$}
\WHILE{$C \ne \emptyset$}
\STATE{Choose $i\in C$ arbitrarily}
\STATE{Create a cluster $R_k \gets \{j:\norm{\x_i-\x_j}\le\mu^{3/4}\}$ (including $i$ itself)}
\STATE{Delete all these points in $R_k$ from $C$}
\STATE{$k \gets k +1$}
\ENDWHILE
\RETURN candidate clusters $\{R_1, R_2, \ldots, R_{K'}\}$
\end{algorithmic}
\end{algorithm}
\begin{algorithm}
\caption{Find clusters (B)}
\label{alg:find_clusterB}
\begin{algorithmic}
\STATE{Define node set $V \gets \{1,\ldots, n\}$, edge set $E \gets \emptyset$, graph $G \gets (V, E)$}
\FOR{$i=1,2,\ldots,n, j = i+1, i+2, \ldots n$}
\STATE{Compute $\v_{ij} = \mathrm{prox}_{\lambda/\nu \norm{\cdot}} (\x_i - \x_j - \nu^{-1} \bdelta_{ij})$ for all $i \ne j$}
\STATE{If $\v_{ij} = \bz$, add $(i,j)$ to the edge set $E$}
\ENDFOR
\STATE{Find all connected components $\{R_1, R_2, \ldots, R_{K'}\}$ of $G$}
\RETURN candidate clusters $\{R_1, R_2, \ldots, R_{K'}\}$
\end{algorithmic}
\end{algorithm}

Our algorithm terminates if the clustering test succeeds, or if the maximum number of iterations is reached. In the algorithm, the code tests for clustering every $t$ iterations of the ADMM solver. The value of $t$ is taken to be 8 in our experiment. At the end of every $t$ iterations, the solver yields a primal solution and a dual solution, from which our algorithm constructs a primal and dual feasible pair for the SOCP formulation by \eqref{eq:updated_feasible_sol}. With the feasible solution, the algorithm then creates candidate clusters, computes the duality gap and constructs the CGR subgradients. The code checks for the CGR subgradient condition and separation condition. If both conditions hold, the clustering test reports `success'. Otherwise, the code runs $t$ more iterations of the ADMM solver and repeats the clustering test. Each iteration of the ADMM solver is of complexity $O(n^2d)$.

To assess the performance of recovery, we employ
a modification of
the Rand index \cite{Rand}.
Let $\chi_i\in \{1,\ldots,K\}$ denote the cluster assignment of node $i$, $i=1,\ldots,n$ in the first clustering and $\chi_i'$ in the second.
A pair of points $1\le i<j\le n$ scores 1 if
either $\chi_i=\chi_j$ and $\chi'_i=\chi'_j$, or  $\chi_i\ne\chi_j$ and $\chi'_i\ne\chi'_j$.  If not (i.e., $\chi_i=\chi_j$ and $\chi'_i\ne \chi'_j$ or vice versa), the pair scores 0.
Then Rand index is defined as this score divided by $n(n-1)/2$, so that 1 means perfect agreement between clusterings. Recall that it is possible for our algorithm to terminate when the clustering test fails. We label the candidate clusters that fail CGR or separation conditions as inconclusive clusters. Points of inconclusive clusters are
scored as 0 in all of their pairs. In contrast to the ordinary Rand index, for which a score of 0.5 means complete failure in the case of two clusters (i.e., no better than random guessing), our modified Rand index can be as low as 0 if all the points are marked as inconclusive.

In the first experiment, we apply the sum-of-norms clustering with multiplicative weights \eqref{eq:son-clustering} to a simulated dataset of two half moons with 500 instances. The angle $\theta$ of each half moon follows a uniform distribution from the interval $[-\frac{\pi}{2}, \frac{\pi}{2}]$. The weight of each point is assigned to be the pdf of a Gaussian distribution $N(0, \frac{\pi}{5})$ at $\theta$.
A random noise which follows a two-dimensional Gaussian distribution with a mean of 0 and a standard deviation of 0.05 displaces the instances from the moons. Twenty-four linearly spaced values of $\lambda$ are taken from the range $[0.0005, 0.0013]$. The range is determined empirically. Furthermore, the maximum number of iterations is chosen to be 50,000. It took approximately 5.7 hours total on an AMD Opteron Processor 6376 to complete the experiment.

\begin{figure}[h]
    \centering
    \begin{subfigure}{.49 \textwidth}
    \centering
    \includegraphics[scale=0.4]{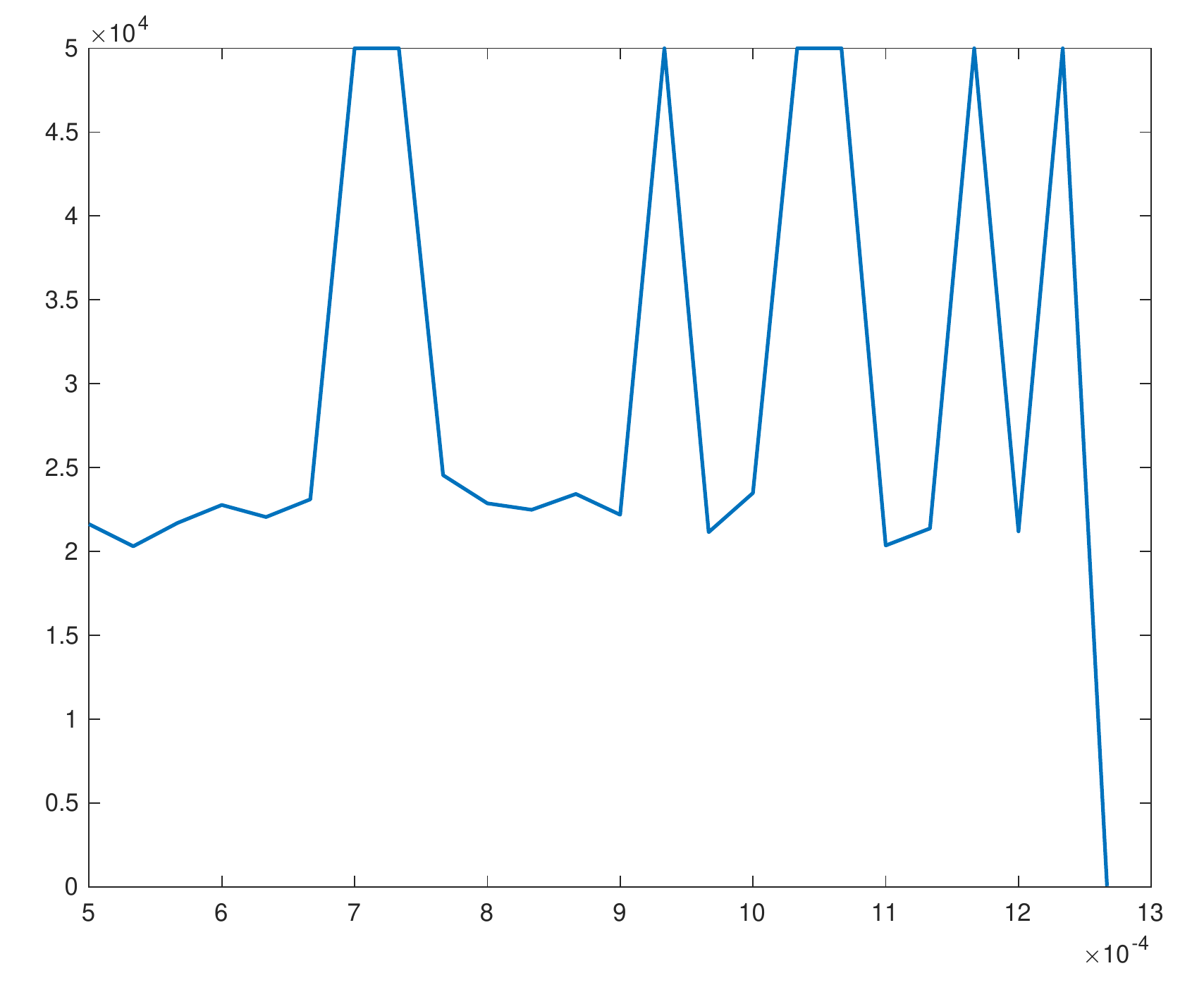}
    \caption{Number of iterations versus $\lambda$}
    \label{fig:iter_hm}
    \end{subfigure}
    \begin{subfigure}{.49 \textwidth}
    \centering
  \includegraphics[scale=0.4]{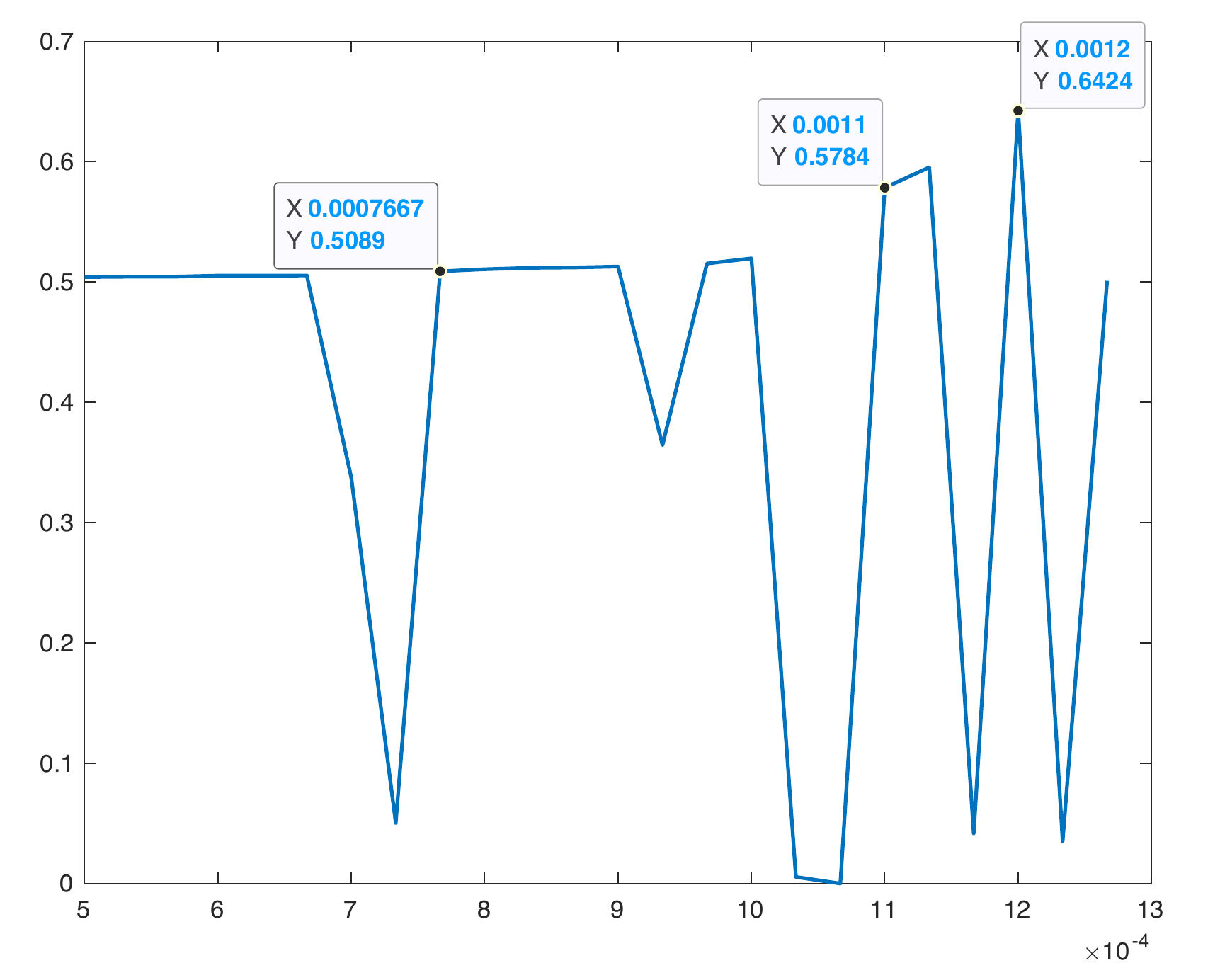}
    \caption{Rand index versus $\lambda$}
    \label{fig:rand_index_hm}
    \end{subfigure}
\end{figure}

Our first objective is to evaluate the performance of our clustering test. At 17 out of 24 values of $\lambda$, the clustering test succeeds before the maximum number of iterations is reached. When $\lambda$ is in the range between $\lambda = 0.001033$ and $\lambda = 0.001067$, the algorithm repeatedly reaches the iteration threshold before the test succeeds as shown in Figure \ref{fig:iter_hm}. The performance is interpretable with theories discussed earlier. The clustering test is not guaranteed to succeed when $\lambda$ is at a fusion value, and the test performs poorly near a fusion value as shown in Figure \ref{fig:iter_hm}. When $n = 500$, there are at most 500 fusion values. All fusion values are in the range between the chosen range of $\lambda$ as observed in the experiment. Hence, fusion occurs frequently, and massive fusion values are located densely in a small region. Thus, in our experiment, it is very likely that the $\lambda$ we pick is near or at a fusion value, which leads to the poor performance of our clustering test at some values of $\lambda$.



The experiment also attempts to explore the relationship between $\lambda$ value and the recovery of half moons. To evaluates the recovery, we compute the Rand index with the recovered clustering and the generative clustering. Figure \ref{fig:rand_index_hm} shows Rand index against $\lambda$ values. The value of Rand index increases and peaks at $\lambda = 0.0012$, where the clustering test succeeds and the Rand index achieves a value of 0.6424.
A visual inspection of the result
(see Section~\ref{sec:SM_Z}) shows that all the points in the middle
of both half-moons are correctly labeled.


In the second experiment, we solve the sum-of-norms clustering problem with multiplicative weights \eqref{eq:multweight} on a mixture of Gaussians.
Multiple researchers have studied the recovery of a mixture of Gaussians using sum-of-norms clustering. Tan and Witten \cite{TanWitten} and Jiang, Vavasis and Zhai \cite{jiangvavasiszhai} performed experiments with unit-weight sum-of-norms clustering \eqref{eq:son-clustering}. They found that \eqref{eq:son-clustering} fails to recover a mixture of Gaussians when the means are close. With carefully chosen weights, the following experiment illustrates that multiplicative weights recover a mixture of Gaussians even when the means are close.

\begin{figure}[h]
    \centering
    \begin{subfigure}{.49 \textwidth}
    \centering
    \includegraphics[scale=0.4]{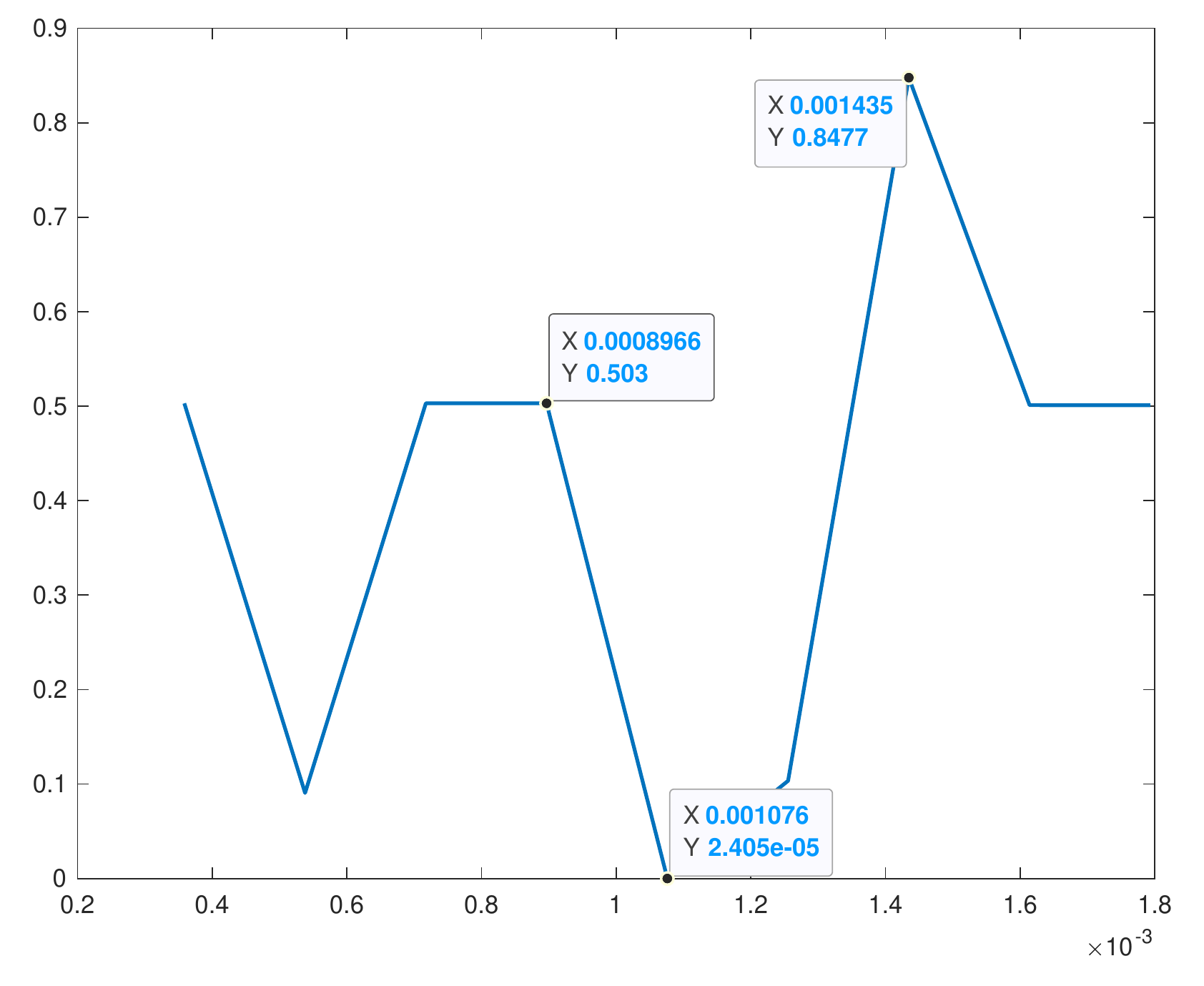}
    \caption{Rand index versus $\lambda$ (all 500 samples)}
    \label{fig:rand_index_500}
    \end{subfigure}
    \begin{subfigure}{.49 \textwidth}
    \centering
  \includegraphics[scale=0.408]{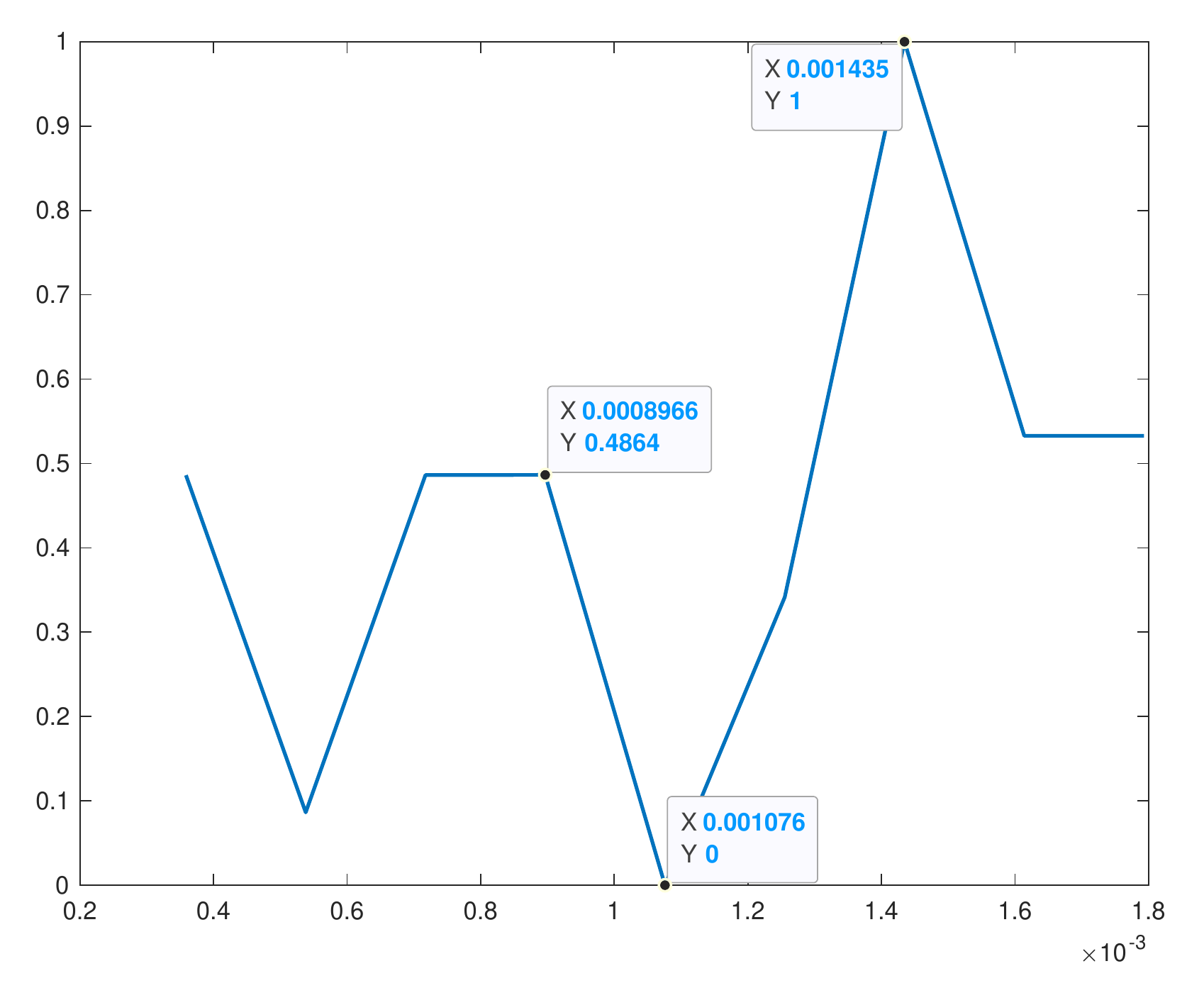}
    \caption{Rand index versus $\lambda$ (105 samples)}
    \label{fig:rand_index_105}
    \end{subfigure}
\end{figure}

In our experiment, we draw 500 samples from a mixture of two Gaussians in $\mathbb R^2$ with equal probability. The two Gaussian means are placed 4 standard deviations apart. We assign the Gaussian pdf value at each sample to be its weight.

We then implement the method described above to solve \eqref{eq:multweight}. From Figure \ref{fig:rand_index_500}, the Rand index peaks at $\lambda = 0.001435$, achieving a value of $0.8477$. The Rand index is at its lowest when $\lambda = 0.001074$. The low Rand index is due to the failure of clustering test for many candidate clusters near fusion values. At other values of $\lambda$, our clustering test succeeds after a finite number of iterations.

Since we are interested in the dataset of a mixture of Gaussians with close means, many points associated with one mean could be placed arbitrarily closer to another mean. Thus, we are more concerned about the recovery of samples that are not far away from their respective centroids. Hence, we also compute the Rand index for samples that are located within 0.7 standard deviations from their respective means. According to Figure \ref{fig:rand_index_105}, there are 105 such samples and the Rand index peaks at $\lambda = 0.001434$ with a value of $1.0$. It took approximately 21.7 hours total on an AMD Opteron Processor 6376 to complete the experiment.

Our first experiment used multiplicative weights to find clusters for a half-moon dataset with uniformly distributed angles.
Our construction of multiplicative weights recovered only the center portion of
each half-moon.
We hypothesize that it was because the distribution does not have a single peak but rather a distributed peak.  Unit weights were also tested but failed to find clusters. It is an open question whether a
more
successful multiplicative-weight formulation exists for this data. Our second experiment adopted multiplicative weights for a mixture of Gaussians dataset. SON with carefully chosen weights successfully recover a mixture of Gaussians even when the means of Gaussians were close. In contrast, SON with unit weights fail to identify clusters when the means are not well separated  \cite{jiangvavasiszhai}\cite{TanWitten}. Sun et al.\ \cite{dsun1} succeed for both half-moon and a mixture of Gaussians datasets with exponentially decaying weights.


Assigning exponentially decaying weights
implicitly imposes a prior hypothesis that the nearest-neighbor structure corresponds to true clustering, which is certainly the case for the standard half-moon dataset. Chi and Lange \cite{ChiLange} assess the effect of the number of nearest neighbors $k$ and the parameter $\phi$ on SON clustering with numerical experiments on a half-moon dataset of 100 points. Setting $k=10$ and $\phi = 0.5$ yields the best clustering. Choosing $k=50$ and $\phi = 0$ results in a similar clustering pattern to our experiment: clusters only form until late then all points quickly coalesce to one cluster. At any value of $\lambda$, SON clustering could not identify two half moons with high accuracy. When $k=10$ and $\phi=0$, or $k=50$ and $\phi = 0.5$, SON clustering correctly identifies clusters for the easier points but fails to cluster points located at the lower tip of the right moon and the upper tip of the left moon.






\section{Discussion}
We proposed a test to certify all clusters obtained from an approximate solution yielded by any primal-dual type method. If the test reports `success', then the clusters are correctly identified. Moreover, if a primal-dual path following method that maintains close proximity to the central path is used, the test is guaranteed to report `success' after a finite number of iterations at non-fusion values of $\lambda$, where strict complementarity holds. A few natural questions concerning strict complementarity and the test itself are (1) Is there a rigorous test that works when strict complementarity fails? (2) What is the complexity of our clustering test since it depends on the choice of $\lambda$ values? (3) Is the test guaranteed to work for a general primal-dual algorithm? (4) Can one 
certify
clusters 
from a primal-only algorithm?

We have also shown the power of sum-of-norms clustering with multiplicative weights. Recall that \eqref{eq:multweight} is understood as a version of \eqref{eq:son-clustering} in which data points are repeated, possibly fractionally many times. Hence, multiplicative weights
inherit
strong theoretical properties from unit weights and guarantee low computational complexity of ADMM. Numerical experiments have demonstrated the strong recovery power of multiplicative weights. We reweighted nodes based on the true distribution of the data. SON with multiplicative weights successfully recovered the mixture of Gaussians.
When such prior information about data distribution is absent, the weight generation remains an open question.

\bibliographystyle{plain}
\bibliography{references}
\appendix
\section{Constructing an SOCP feasible solution}
\label{sec:SM_A}
In this section, we first derive the Lagrangian dual of \eqref{eq:multweight}.
The dual variables are $\bdelta_{ij}\in \R^d$ for
$1\le i<j\le n$.   Then we show
how to define additional primal and dual variables to come up with
feasible points for
\eqref{eq:mwsocp_primal} and \eqref{eq:mwsocp_dual}
given $(\x,\bdelta)$.
Rewrite the original problem \eqref{eq:multweight}
with constraints as
$$
\begin{array}{ll}
  \min_{\x,\y} & \frac{1}{2}\sum_{i=1}^nr_i\Vert\x_i-\a_i\Vert^2
  + \lambda\sum_{i<j}r_ir_j\Vert \y_{ij}\Vert \\
  \mbox{s.t.} & r_ir_j(\x_i-\x_j-\y_{ij})=\bz \quad \forall i,j \mbox{ s.t. }1\le i < j\le n.
\end{array}
$$
Then we introduce Lagrange multipliers $\bdelta_{ij}\in\R^d$
for all $1\le i<j\le n$
to rewrite the constrained problem in saddle-point form:
$$
\min_{\x,\y}\max_{\bdelta}
\frac{1}{2}\sum_{i=1}^n r_i\Vert \x_i-\a_i\Vert^2 +
\lambda\sum_{1\le i<j\le n}r_ir_j\Vert\y_{ij}\Vert -
\sum_{1\le i<j\le n}r_ir_j\bdelta_{ij}^T(\x_i-\x_j-\y_{ij}).
$$
The Lagrangian dual is therefore
$$\max_{\bdelta}
\min_{\x,\y}
\frac{1}{2}\sum_{i=1}^n r_i\Vert \x_i-\a_i\Vert^2 +
\lambda\sum_{1\le i<j\le n}r_ir_j\Vert\y_{ij}\Vert -
\sum_{1\le i<j\le n}r_ir_j\bdelta_{ij}^T(\x_i-\x_j-\y_{ij}),
$$
which we now proceed to simplify.
Observe that if $\Vert\bdelta_{ij}\Vert>\lambda$ for any $1\le i<j\le n$ then
the inner min problem is unbounded (taking the corresponding
$\y_{ij}=\mu\bdelta_{ij}$ and letting $\mu\rightarrow-\infty$), so
we obtain the hidden constraint $\Vert\bdelta_{ij}\Vert\le \lambda$.
If $\Vert\bdelta_{ij}\Vert<\lambda$, then the optimal choice for $\y_{ij}$
is $\y_{ij}=\bz$.  If $\Vert\bdelta_{ij}\Vert=\lambda$, then any $\y_{ij}$
of the form $\mu\bdelta_{ij}$, $\mu\le 0$, is optimal.  In either case, at
the inner optimizer the terms involving $\y_{ij}$ cancel out, thus
leaving:
$$\max_{\bdelta: \Vert{\bdelta_{ij}}\Vert\le\lambda}
\min_{\x}
\frac{1}{2}\sum_{i=1}^n r_i\Vert \x_i-\a_i\Vert^2
-
\sum_{1\le i<j\le n}r_ir_j\bdelta_{ij}^T(\x_i-\x_j)
$$
This objective is separable in the $\x_i$'s; in particular, it is
rewritten as $S_1+\cdots+S_n+{\rm const}$ where
$$S_i=\frac{r_i}{2}\left[\x_i^T\x_i-2\left(\a_i+\sum_j r_j\bdelta_{\abrack{ij}}\right)^T\x_i\right],$$
where we have used subscript notation introduced in
\eqref{eq:abrack_notation}.  Clearly the optimizing choice of $\x_i$ is the quantity in
parentheses, thus yielding the dual
problem
$$\max_{\bdelta}\{h'(\bdelta):
\Vert\bdelta_{ij}\Vert\le \lambda\>\forall i,j\mbox{ s.t. }1\le i<j \le n\}$$ where
\begin{align*}
 h'(\bdelta) &:=
\frac{1}{2}\sum_{i=1}^n
r_i\left\Vert\sum_{j=1}^nr_j\bdelta_{\abrack{ij}}\right\Vert^2-
\sum_{1\le i<j\le n}r_ir_j\bdelta_{ij}^T\left(
\a_i - \a_j +
\sum_{k=1}^nr_k\bdelta_{\abrack{ik}}-\sum_{k=1}^nr_k\bdelta_{\abrack{jk}}\right)\\
&=
\frac{1}{2}\sum_{i=1}^n
r_i\left\Vert\sum_{j=1}^nr_j\bdelta_{\abrack{ij}}\right\Vert^2-
\frac{1}{2}\sum_{i,j=1}^nr_ir_j\bdelta_{\abrack{ij}}^T\left(
\a_i - \a_j +
\sum_{k=1}^nr_k\bdelta_{\abrack{ik}}-\sum_{k=1}^nr_k\bdelta_{\abrack{jk}}\right),
\end{align*}
which we write as $T_1+T_2+T_3+T_4$ and analyze the terms
separately.

First,
$$T_1=\frac{1}{2}\sum_{i=1}^n
r_i\left\Vert\sum_{j=1}^nr_j\bdelta_{\abrack{ij}}\right\Vert^2.$$
Next,
\begin{align*}
  T_2 &= -\frac{1}{2} \sum_{i,j=1}^nr_ir_j\bdelta_{\abrack{ij}}^T(\a_i-\a_j) \\
  &= -\frac{1}{2}\sum_{i,j=1}^nr_ir_j\bdelta_{\abrack{ij}}^T\a_i+
  \frac{1}{2}\sum_{i,j=1}^nr_ir_j\bdelta_{\abrack{ij}}^T\a_j  \\
  &= \frac{1}{2}\sum_{i,j=1}^nr_ir_j\bdelta_{\abrack{ji}}^T\a_i+
  \frac{1}{2}\sum_{i,j=1}^nr_ir_j\bdelta_{\abrack{ij}}^T\a_j  \\
  &= \frac{1}{2}\sum_{i,j=1}^nr_ir_j\bdelta_{\abrack{ij}}^T\a_j+
  \frac{1}{2}\sum_{i,j=1}^nr_ir_j\bdelta_{\abrack{ij}}^T\a_j  \\
  &=\sum_{i,j=1}^nr_ir_j\bdelta_{\abrack{ij}}^T\a_j.
\end{align*}

Next,
\begin{align*}
  T_3 &=-\frac{1}{2}\sum_{i,j=1}^nr_ir_j
  \bdelta_{\abrack{ij}}^T\sum_{k=1}^nr_k\bdelta_{\abrack{ik}} \\
  &=-\frac{1}{2}\sum_{i=1}^nr_i\left(\sum_{j=1}^nr_j
  \bdelta_{\abrack{ij}}^T\right)
  \left(\sum_{k=1}^nr_k\bdelta_{\abrack{ik}}\right) \\
  &=-\frac{1}{2}\sum_{i=1}^nr_i \left\Vert \sum_{j=1}^nr_j
  \bdelta_{\abrack{ij}}\right\Vert^2 \\
  &=-T_1.
\end{align*}
Thus, $T_1$ and $T_3$ cancel, leaving only $T_2$ and $T_4$.
Finally,
\begin{align*}
  T_4 &=\frac{1}{2}\sum_{i,j=1}^nr_ir_j
  \bdelta_{\abrack{ij}}^T\sum_{k=1}^nr_k\bdelta_{\abrack{jk}} \\
  &=\frac{1}{2}\sum_{j=1}^nr_j
  \left(\sum_{i=1}^nr_i\bdelta_{\abrack{ij}}\right)
  \left(\sum_{i=1}^nr_k\bdelta_{\abrack{jk}}\right) \\
  &=-\frac{1}{2}\sum_{j=1}^nr_j
  \left(\sum_{i=1}^nr_i\bdelta_{\abrack{ij}}\right)
  \left(\sum_{i=1}^nr_k\bdelta_{\abrack{kj}}\right) \\
  &=-\frac{1}{2}\sum_{j=1}^nr_j
  \left\Vert\sum_{i=1}^n r_i\bdelta_{\abrack{ij}}\right\Vert^2.
\end{align*}
Thus,
\begin{align*}
  h'(\bdelta) &= T_4+T_2 \\
  &=
  -\frac{1}{2}\sum_{j=1}^nr_j
  \left\Vert\sum_{i=1}^n r_i\bdelta_{\abrack{ij}}\right\Vert^2
  +\sum_{i,j=1}^nr_ir_j\bdelta_{\abrack{ij}}^T\a_j
\end{align*}
arriving
at the following Lagrangian
dual of \eqref{eq:multweight}:
\begin{subequations}
\label{eq:son_dual}
\begin{align}
  \underset{\bdelta}{\text{max}}
       & \quad h'(\bdelta) =  -\frac{1}{2}\sum_{j=1}^nr_j
  \left\Vert\sum_{i=1}^n r_i\bdelta_{\abrack{ij}}\right\Vert^2
  +\sum_{i,j=1}^nr_ir_j\bdelta_{\abrack{ij}}^T\a_j\label{eq:d_obj_admm} \\
   \text{s.t}
       & \quad \norm{\bdelta_{ij}} \le \lambda \;,\quad \forall 1 \le i<j \le n \;. \label{eq:d_constr1_admm}
\end{align}
\end{subequations}


Let $(\x, \bdelta)$ denote the output yielded by the primal-dual algorithm solving for
\eqref{eq:multweight}. To construct a feasible solution from $(\x, \bdelta)$, we first update $\bdelta$ as follows
\[
\bdelta_{ij} \leftarrow
\left\{
\begin{array}{ll}
  \frac{\lambda \bdelta_{ij}}{\norm{\bdelta_{ij}}}, \quad & \mbox{if } \norm{\bdelta_{ij}} > \lambda, \\
  \bdelta_{ij}, & \mbox{otherwise}.
\end{array}
\right.
\]
The updated $\bdelta_{ij}$ has norm no more than $\lambda$. Notice that the perturbation is small provided that the dual solution was already close to the feasible set.

 Next, define the following variables:
\begin{equation}
\begin{aligned}
    \y_{ij} = \x_i - \x_j, \quad &\forall 1 \le i < j \le n,\\
    \z_{i} = \x_i - \a_i, \quad &\forall i = 1, \dots, n,\\
    s_i = \frac{1}{2} (1 + \|\z_{i}\|^2), \quad &\forall i = 1, \dots, n,\\
    u_i = \frac{1}{2} (-1 + \|\z_{i}\|^2), \quad &\forall i= 1, \dots, n,\\
    t_{ij} = \|\y_{ij}\|, \quad &\forall 1 \le i < j \le n,\\
    \bbeta_i = \sum_{j=1}^nr_i\bdelta_{\abrack{ij}}, \quad &\forall i = 1, \dots, n,\\
    \gamma_i = \frac{1}{2}(1 - \|\bbeta_i\|^2), \quad &\forall i = 1, \dots, n.
\end{aligned}
\label{eq:updated_feasible_sol}
\end{equation}
It can be easily verified that these newly defined variables \eqref{eq:updated_feasible_sol} are feasible for \eqref{eq:mwsocp_primal} and \eqref{eq:mwsocp_dual}.
Furthermore, one checks that the objectives for
\eqref{eq:mwsocp_primal} and \eqref{eq:mwsocp_dual} are exactly the constant factor
$\frac{1}{2}\sum_{i=1}^nr_i$ larger than $f'(\x)$, $h'(\bdelta)$ respectively.
Therefore, the duality gap and hence the nearness to optimality is preserved.

%

\section{Finding a bound on residuals of complementary slackness}
\label{sec:SM_B_mw}
At the optimizer, there hold $\bepsilon = \bz, \bsigma = \bz$ by KKT conditions. The system of equalities above becomes the complementary slackness condition. At an approximate solution, the right-hand sides $\bepsilon, \bsigma$ are non-zero. If $\bepsilon^{ij} = \begin{pmatrix}
\mu'\\
\bz
\end{pmatrix}, \bsigma^i = \begin{pmatrix}
\mu'\\
\bz\\
0
\end{pmatrix}$ for all $i = 1, \dots, n$ and for all $1 \le i < j \le n$, we refer the corresponding solution as a $\mu'$-centered solution. Otherwise, an upper bound on general right-hand sides $\bepsilon, \bsigma$ can be derived from the duality gap:

\begin{align*}
& f'(\x) - h'(\bdelta)\\
= &f(\x, \y, \z, s, u, t) - h(\bdelta, \bbeta, \gamma)\\
= &\sum_{i=1}^n r_i s_i - \sum_{i=1}^n r_i \gamma_i + \sum_{i=1}^n r_i \langle \x_i - \a_i, \bbeta_i \rangle + \lambda \sum_{1 \le i < j \le n} r_i r_j t_{ij} - \sum_{i=1}^n r_i \langle \x_i, \bbeta_i \rangle \\
&\text{(By adding and subtracting} \sum_{i=1}^n  r_i \langle \x_i, \bbeta_i \rangle)\\
= &\sum_{i=1}^n r_i s_i - \sum_{i=1}^n r_i \gamma_i + \sum_{i=1}^n r_i \langle \x_i - \a_i, \bbeta_i \rangle + \lambda \sum_{1 \le i < j \le n} r_i r_j t_{ij} - \sum_{i=1}^n  r_i \left\langle \x_i, \sum_{j=1}^n r_i \bdelta_{\abrack{ij}} \right\rangle \\
& \text{(By \eqref{eq:mwd_constr1})}\\
= &\sum_{i=1}^n r_i s_i - \sum_{i=1}^n r_i \gamma_i + \sum_{i=1}^n r_i \langle \x_i - \a_i, \bbeta_i \rangle + \lambda \sum_{1 \le i < j \le n} r_i r_j t_{ij} - \sum_{1 \le i < j \le n} r_i r_j \langle \x_j - \x_i, \bdelta_{ij} \rangle\\
&\text{(By expanding the summation)}\\
= & \sum_{i=1}^n r_i (s_i - \gamma_i + \langle \x_i - \a_i, \bbeta_i \rangle) + \sum_{1 \le i < j \le n} r_i r_j ( \lambda t_{ij} + \langle \y_{ij}, \bdelta_{ij} \rangle) \quad \text{(By \eqref{eq:mwp_constr1})}\\
= & \sum_{i=1}^n r_i (s_i (1 - \gamma_i) + \langle \z_i, \bbeta_i \rangle + u_i \gamma_i) + \sum_{1 \le i < j \le n} r_i r_j ( \lambda t_{ij} + \langle \y_{ij}, \bdelta_{ij} \rangle) \quad \text{(By \eqref{eq:mwp_constr2}, \eqref{eq:mwp_constr3})}\\
=& \sum_{i=1}^n r_i \sigma_1^i + \sum_{1 \le i < j \le n} r_i r_j \epsilon_1^{ij}
\end{align*}
Each term in the both summations is non-negative as shown below:
\[
\sigma_1^i = s_i (1 - \gamma_i) + \langle \z_i, \bbeta_i \rangle + u_i \gamma_i = \frac{1}{2} (\norm{\z_i}^2 + 2 \langle \z_i, \bbeta_i \rangle + \norm{\bbeta_i}^2) \ge 0, \quad \forall i = 1, \dots, n,
\]
\[
\epsilon_1^{ij} = \lambda t_{ij} + \langle \y_{ij}, \bdelta_{ij} \rangle \ge \lambda t_{ij} - \norm{\y_{ij}} \norm{\bdelta_{ij}} \ge \lambda t_{ij} - \lambda \norm{\y_{ij}} \ge 0, \quad \forall 1 \le i < j \le n.
\]
Define $\mu:= f'(\x) - h'(\bdelta)$ to be the duality gap at the feasible solution. Combined with the non-negativity condition, $\sigma_1^i, \epsilon_1^{ij}$ satisfy $\sigma_1^i \le \mu$ for all $i = 1, \dots, n$ and $\epsilon_1^{ij} \le \mu$ for all $1\le i<j \le n$. At termination, the duality gap $\mu$ at the feasible solution is small, which implies the right-hand sides $\sigma_1^i, \epsilon_1^{ij}$ are also well bounded.

We now have $\sigma_1^i, \epsilon_1^{ij}$ upper bounded in terms of $\mu$, and the remainder of the section is to establish upper bounds on $\norm{\bepsilon^{ij}_2}$ and $\norm{\begin{pmatrix}
\bsigma_2^i\\
\sigma_3^i
\end{pmatrix}}$. In fact, in \eqref{eq:mw_bound_epsilon2S} and \eqref{eq:mw_bound_sigma23S} below, we show that both are upper bounded by $O(\sqrt{\mu})$. Consider a general setting of second-order cone programming.

\begin{lemma}
Let $(\x, \z)$ denote a primal and dual feasible solution for a second-order cone program where $\x = \begin{pmatrix}
x_0\\
\bar \x
\end{pmatrix}, \z = \begin{pmatrix}
z_0\\
\bar \z
\end{pmatrix}$. If $\x^T \z \le \mu$, then $\norm{z_0 \bar \x + x_0 \bar \z} \le \sqrt{2x_0 z_0 \mu}$.
\label{lemma:O(sqrt(mu))}
\end{lemma}

\begin{proof}
If $x_0 = 0$, then $\norm{\bar \x} \le x_0 = 0$ by feasibility assumption. Hence, $\bar \x = \bz$, which implies $\norm{z_0 \bar \x + x_0 \bar \z} = 0 \le \sqrt{2x_0 z_0 \mu}$. Similarly, if $z_0 = 0$, then $\z$ satisfies $\norm{z_0 \bar \x + x_0 \bar \z} = 0 \le \sqrt{2x_0 z_0 \mu}$ by the same argument.

Otherwise, $x_0 > 0, z_0 > 0$, and we derive the following inequalities
\begin{align*}
    \x^T \z &= x_0 z_0 + \bar \x^T \bar \z \le \mu\\
    \Rightarrow 1 + \frac{\bar \x^T}{x_0} \frac{\bar \z}{z_0} &\le \frac{\mu}{x_0 z_0} \quad \text{(Since $x_0 > 0, z_0 > 0$)}\\
    \Rightarrow \norm{\frac{\bar \x}{x_0} + \frac{\bar \z}{z_0}}^2 = \norm{\frac{\bar \x}{x_0}}^2 + \norm{\frac{\bar \z}{z_0}}^2 + 2 \frac{\bar \x^T}{x_0} \frac{\bar \z}{z_0} &\le 2 - 2 + \frac{2\mu}{x_0 z_0} \quad \text{(Since $x_0 \ge \norm{\bar \x}, z_0 \ge \norm{\bar \z}$)}\\
    \Rightarrow \norm{\frac{\bar \x}{x_0} + \frac{\bar \z}{z_0}} &\le \sqrt{\frac{2\mu}{x_0 z_0}}\\
    \Rightarrow \norm{z_0 \bar \x + x_0 \bar \z} &\le \sqrt{2x_0 z_0 \mu}.
\end{align*}
\end{proof}

Let $r':= \sum_{l=1}^n r_l$ and $\bar \a := \frac{1}{r'} \sum_{i=1}^n r_i \a_i$ denote the weighted centroid of all data points. Let $\x_1' := \x_2' := ... := \x_n' := \bar \a$. Then the primal objective value of the original sum-of-norms formulation at $\x'$ is
\[
f'(\x') = \frac{1}{2} \sum_{l=1}^n r_l \norm{\bar \a - \a_l}^2.
\]
Let $\bdelta_{ij}' = \bz$ for all $1 \le i < j \le n$. Then $\bdelta'$ is a feasible solution to the dual problem of the original formulation and the dual objective value at $\bdelta'$ is
\[
h'(\bdelta') = 0.
\]
Let $f^*$ and $h^*$ denote the primal and dual optimal values of the SOCP respectively, which must satisfy the following inequality by strong duality:
\[
\frac{r'}{2} = h'(\bdelta') + \frac{r'}{2} \le h^* = f^* \le f'(\x') + \frac{r'}{2} = \frac{1}{2} \sum_{l=1}^n r_l \norm{\bar \a - \a_l}^2 + \frac{r'}{2}.
\]
At the feasible solution $(\x, \y, \z, s, u, t, \bdelta, \bbeta, \gamma)$, the objective value is at a distance of at most $\mu$ away from the optimal value, which implies
\[
\sum_{i=1}^n r_i s_i + \lambda \sum_{1\le i<j \le n}r_i r_j t_{ij} \le f^* + \mu \le \frac{1}{2} \sum_{l=1}^n r_l \norm{\bar \a - \a_l}^2 + \frac{r'}{2} + \mu,
\]
which is rearranged to
\[
\sum_{i=1}^n r_i \left (s_i - \frac{1}{2} \right) + \lambda \sum_{1\le i<j \le n} r_i r_j t_{ij} \le \frac{1}{2} \sum_{l=1}^n r_l \norm{\bar \a - \a_l}^2 + \mu.
\]
Moreover, by feasibility, $s_i \ge \frac{1}{2}$ holds for all $i = 1, \dots, n$ and $t_{ij} \ge 0$ holds for all $1 \le i < j \le n$. Hence,
\[
s_i \le \frac{1}{r_i} \left (\frac{1}{2}\sum_{l=1}^n r_l \norm{\bar \a - \a_l}^2 + \mu \right)+ \frac{1}{2} , \quad t_{ij} \le \frac{1}{\lambda r_i r_j} \left (\frac{1}{2} \sum_{l=1}^n r_l \norm{\bar \a - \a_l}^2 + \mu \right).
\]
As $t_{ij} \lambda + \y_{ij}^T \bdelta_{ij} = \epsilon_1^{ij} \le \mu$, $\norm{\bepsilon_2^{ij}}$ has the following upper bound by Lemma \ref{lemma:O(sqrt(mu))}
\begin{equation}
    \norm{\bepsilon_2^{ij}} = \norm{t_{ij} \bdelta_{ij} + \lambda \y_{ij}}
\le \sqrt{2 t_{ij} \lambda \mu} \le \sqrt{\frac{1}{r_i r_j} \left (\sum_{l=1}^n r_l \norm{\bar \a - \a_l}^2 \mu+ 2\mu^2 \right)}.
\label{eq:mw_bound_epsilon2S}
\end{equation}
Similarly, at the feasible solution, the dual objective value is at a distance of at most $\mu$ away from the optimal value, which implies
\[
\sum_{i=1}^n r_i \a_i^T \bbeta_i + \sum_{i=1}^n r_i \gamma_i \ge h^* - \mu \ge \frac{r'}{2}- \mu,
\]
which is rearranged to
\[
 \sum_{i=1}^n r_i \left (\frac{1}{2} - \gamma_i \right) \le \sum_{i=1}^n r_i \a_i^T \bbeta_i + \mu.
\]
By feasibility, $\frac{1}{2} - \gamma_i \ge 0$, which implies
\[
1 - \gamma_i \le \frac{1}{2} + \frac{1}{r_i} \left (\sum_{l=1}^n r_l \a_l^T \bbeta_l + \mu \right ).
\]
Since $\lambda \ge \norm{\bdelta_{ij}}$, $\norm{\bbeta_i}$ satisfies
\[
\norm{\bbeta_i} = \norm{\sum_{j \ne i} r_j \bdelta_{\abrack{ik}}} \le (r' - r_i) \lambda.
\]
By Cauchy-Schwartz inequality,
\[
\a_i^T \bbeta_i \le \norm{\a_i} \cdot \norm{\bbeta_i} \le (r' - r_i) \lambda \norm{\a_i}.
\]
Therefore, $1 - \gamma_i$ satisfies
\[
1 - \gamma_i \le \frac{1}{2} + \frac{1}{r_i} \left (\sum_{l=1}^n r_l (r' - r_i) \lambda \norm{\a_l} + \mu \right ).
\]
Since $s_{i} (1 - \gamma_i) + \z_{i}^T \bbeta_{i} + u_i \gamma_i = \sigma_1^{i}$,
$\norm{\begin{pmatrix}
\bsigma_2^i\\
\sigma_3^i
\end{pmatrix}}$ has the following upper bound by Lemma \ref{lemma:O(sqrt(mu))}
\begin{equation}
\begin{aligned}
  &  \norm{\begin{pmatrix}
\bsigma_2^i\\
\sigma_3^i
\end{pmatrix}} \\
=& \norm{\begin{pmatrix}
s_{i} \bbeta_{i} + (1 - \gamma_{i}) \z_{i}\\
s_{i} \gamma_i + (1 - \gamma_{i}) u_i
\end{pmatrix}}\\
\le& \sqrt{2 s_i (1 - \gamma_i) \mu} \\
=& \sqrt{2 \cdot \left (\frac{1}{r_i} \left (\frac{1}{2}\sum_{l=1}^n r_l \norm{\bar \a - \a_l}^2 + \mu \right)+ \frac{1}{2} \right ) \cdot \left (\frac{1}{2} + \frac{1}{r_i} \left (\sum_{l=1}^n r_l (r' - r_i) \lambda \norm{\a_l} + \mu \right ) \right ) \cdot \mu}\\
=& \sqrt{\left (\frac{1}{r_i}\sum_{l=1}^n r_l \norm{\bar \a - \a_l}^2 + \frac{2 \mu}{r_i} + 1 \right) \cdot \left (\frac{1}{2} + \frac{1}{r_i} \left (\sum_{l=1}^n r_l (r' - r_i) \lambda \norm{\a_l} + \mu \right ) \right ) \cdot \mu}.
\end{aligned}
\label{eq:mw_bound_sigma23S}
\end{equation}

\section{Proof of Lemma \ref{cluster_lemma_subgradient}}
\label{sec:SM_2}

We restate Lemma \ref{cluster_lemma_subgradient} as follows:
\begin{lemma}
For all $i, j \in C, i < j$,
$\q_{ij}$
as defined in \eqref{eq:q_def} satisfies
\begin{align}
  \a_i-\frac{1}{r'} \sum_{l\in C}r_l \a_l &= \sum_{j\in C} r_j \q_{\abrack{ij}}, \quad \forall i \in C
\end{align}
where $r' = \sum_{i \in C} r_i$.
\end{lemma}



\begin{proof}
Substitute the primal constraint \eqref{eq:mwp_constr3} into the perturbed complementary slackness \eqref{eq:mwcs_b3} to obtain the following equality of $\gamma_i$ and $s_i$
\[
1 - \gamma_i = s_i - \sigma_3^i, \quad \forall i = 1, \dots, n.
\]
Substitute the equality above into \eqref{eq:mwcs_b2} and divide both sides by $s_i$ to obtain the following equation of $\bbeta_i$ in terms of $\z_i$
\begin{equation}
  \bbeta_{i} = - \z_i + \bomega_i, \quad \forall i = 1, \dots, n.
  \label{eq:mwopt_condp1}
\end{equation}
Notice that the operation is valid because $s_i \ge \frac{1}{2}$ by the primal constraint \eqref{eq:mwp_constr3} and \eqref{eq:mwp_constr5}. Substitute the primal constraint \eqref{eq:mwp_constr2} and the equality above into the dual constraint \eqref{eq:mwd_constr1} yielding the following equality
with the definition of $\bdelta_{\abrack{ij}}$, the equality \eqref{eq:mwopt_condp1} is rewritten as
\begin{equation}
- \x_i + \a_i + \bomega_i + \sum_{j=1}^nr_j \bdelta_{\abrack{ij}} = \bz, \quad \forall i = 1, \dots, n.
\label{eq:mwopt_condp2}
\end{equation}
Multiply \eqref{eq:mwopt_condp2} by $r_i$ and sum them over all $i \in C$ and divide the new equality by
$r'$ to obtain
\begin{equation}
    - \frac{1}{r'} \sum_{i \in C} r_i \x_i + \frac{1}{r'}\sum_{i\in C} r_i \a_i + \frac{1}{r'} \sum_{i \in C} r_i \bomega_i  + \frac{1}{r'} \sum_{i \in C} \sum_{k \notin C} r_i r_k \bdelta_{\abrack{ik}} = \bz.
    \label{eq:mwopt_condp3}
\end{equation}
Change the index in \eqref{eq:mwopt_condp3} from $i$ to $j$. Subtract \eqref{eq:mwopt_condp3} from \eqref{eq:mwopt_condp2} to obtain
\begin{align*}
  &- \x_i + \frac{1}{r'} \sum_{j \in C} r_j \x_j + \a_i - \frac{1}{r'}\sum_{l\in C} r_l \a_l + \bomega_i - \frac{1}{r'}  \sum_{j \in C} r_j \bomega_j\\
  &+ \sum_{j \in C} r_j \bdelta_{\abrack{ij}} + \sum_{k \notin C} r_k \bdelta_{\abrack{ik}} - \frac{1}{r'} \sum_{k \notin C} r_k  \sum_{j \in C} r_j \bdelta_{\abrack{jk}} = \bz, \quad \forall i \in C,
\end{align*}
which is rearranged to
\begin{align*}
     &\a_i - \frac{1}{r'}\sum_{l\in C} r_l \a_l \\
     =& \x_i - \frac{1}{r'} \sum_{j \in C} r_j \x_j - \bomega_i + \frac{1}{r'} \sum_{j \in C} r_j \bomega_j - \sum_{j \in C} r_j \bdelta_{\abrack{ij}} - \frac{1}{r'} \sum_{j \in C} \sum_{k \notin C} r_k r_j (\bdelta_{\abrack{ik}} - \bdelta_{\abrack{jk}})\\
     =& \sum_{j \in C } \left [\frac{r_j}{r'} (\x_i - \x_j - \bomega_i + \bomega_j) - r_j \bdelta_{\abrack{ij}} - \frac{r_j}{r'} \sum_{k \notin C}  r_k (\bdelta_{\abrack{ik}} - \bdelta_{\abrack{jk}})\right]\\
     &= \sum_{j \in C} r_j \q_{\abrack{ij}} \quad (\text{By definition}), \quad \forall i \in C. \label{eq:opt_cond5}
\end{align*}
Moreover, by the definition of $\q_{ij}$, we observe the following property for all $i, j \in C, i \ne j$
\begin{align*}
    \q_{\abrack{ij}} &= -\bdelta_{\abrack{ij}} + \frac{1}{r'} \cdot (\x_i - \x_j - \bomega_i + \bomega_j) - \frac{1}{r'} \sum_{k \notin C} r_k (\bdelta_{\abrack{ik}} - \bdelta_{\abrack{jk}})\\
    &= \bdelta_{\abrack{ji}} - \frac{1}{r'} \cdot (\x_j - \x_i - \bomega_j + \bomega_i) + \frac{1}{r'} \sum_{k \notin C} r_k (\bdelta_{\abrack{jk}} - \bdelta_{\abrack{ik}}) \\
    &= - \q_{\abrack{ji}}
\end{align*}
\end{proof}

\section{Proof of Lemma \ref{lemma_optimality_socp}}
\label{SM_three_lemmas}

We restate Lemma \ref{lemma_optimality_socp} as follows:
\begin{lemma}
The solution defined by \eqref{eq:feasible_sol} is optimal for SOCP \eqref{eq:mwsocp_primal} and \eqref{eq:mwsocp_dual} at $\lambda$.
\end{lemma}

\begin{proof}
By construction, the primal constraints \eqref{eq:mwp_constr1}, \eqref{eq:mwp_constr2}, \eqref{eq:mwp_constr3}, \eqref{eq:mwp_constr4}, \eqref{eq:mwp_constr5}, the dual constraints \eqref{eq:mwd_constr2}, \eqref{eq:mwd_constr3}, and the complementary slackness conditions \eqref{eq:mwcs_a1}, \eqref{eq:mwcs_a2}, \eqref{eq:mwcs_b1}, \eqref{eq:mwcs_b2} and \eqref{eq:mwcs_b3} with $\bepsilon = \bz, \bsigma = \bz$ are automatically satisfied. It remains to check if the solution satisfies \eqref{eq:mwd_constr1}.\\

\textbf{Verification for \eqref{eq:mwd_constr1}:} For any $i \in C_k$ with some $k \in [K]$, \eqref{eq:mwd_constr1} is rewritten as follows due to \eqref{eq:son-clustering_chiquet} and \eqref{eq:son-clustering_multweight_oc}
\begin{align*}
    &\sum_{j =1}^n r_j \bdelta^*_{\abrack{ij}}+\bbeta_i^* \\
    = &\sum_{j \in C_k}r_j \bdelta'_{\abrack{ij}} + \lambda \sum_{k \ne k'}r_{k'}' \frac{\x_{k'} - \x_k}{\|\x_{k} - \x_{k'}\|}+\a_i-\x_i^*\\
    =& \bar \a_k - \a_i + \lambda \sum_{k \ne k'}r_{k'}' \frac{\x_{k'} - \x_k}{\|\x_{k} - \x_{k'}\|}+\a_i-\x_k\\
    =& \bar \a_k + \lambda \sum_{k \ne k'}r_{k'}' \frac{\x_{k'} - \x_k}{\|\x_{k} - \x_{k'}\|}-\x_k\\
    =& \bz.
\end{align*}

By KKT conditions, the solution defined above forms an optimal primal-dual pair.

\end{proof}

\section{Proof of Lemma \ref{lemma_bdelta}, \ref{lemma_g_diff}, \ref{lemma_bomega} and \ref{lemma_main}}
\label{SM_C}
\underline{Bound $\norm{\bdelta_{ij}}$}

We restate Lemma \ref{lemma_bdelta} as follows:
\begin{lemma}
For all $i,j \in C, i < j$, the following inequality holds
\[
\norm{\bdelta_{ij}} \le \lambda - r + p' \mu
\]
where $r:= \min_{l \ne l', l, l' \in C_k, k \in [K]} (\lambda - \norm{\bdelta^a_{ll'}}) > 0$.
\end{lemma}

\begin{proof}
Let $i, j \in C$ and $i < j$. By the definition of analytic center and strict complementarity,
\[
\norm{\bdelta^a_{ll'}} < \lambda,
\]
holds for all $l < l', l, l' \in C_k, k \in [K]$. Hence, $r >0$ by definition. Moreover, $r$ also satisfies
\[
\norm{\bdelta^a_{ij}} \le \lambda - r, \quad \forall i,j \in C, i < j.
\]
Since $\norm{\bdelta_{ij} - \bdelta^a_{ij}} \le p' \mu$, we obtain
\[
\norm{\bdelta_{ij}} \le \lambda - r + p' \mu, \quad \forall i,j \in C, i < j.
\]
\end{proof}

\underline{Bound $\norm{\bdelta_{\abrack{ik}} - \bdelta_{\abrack{jk}}}$}
\label{sec:g_diff}

We restate Lemma \ref{lemma_g_diff} as follows:
\begin{lemma}
For all $i,j \in C$ and $k \notin C$, the following inequality holds
\[
\norm{\bdelta_{\abrack{ik}} - \bdelta_{\abrack{jk}}} \le \frac{4 \lambda p \mu}{q - 2p \mu} +\frac{\left (\sqrt{ \frac{1}{r_i r_k}} + \sqrt{\frac{1}{r_j r_k}}\right ) \cdot \sqrt{\sum_{l=1}^n r_l \norm{\bar \a - \a_l}^2 \mu+ 2\mu^2 }} {q - 2p \mu} + \frac{\mu}{q - 2p \mu}
\]
\end{lemma}
\begin{proof}
Let $i, j \in C$ and $k \notin C$. Without loss of generality, we may assume $i<j<k$. Hence, $\bdelta_{\abrack{ik}}=\bdelta_{ik}, \bdelta_{\abrack{jk}}=\bdelta_{jk}$. By $\eqref{eq:mwcs_a2}$, we derive
\begin{equation}
t_{ik} \bdelta_{ik} - t_{jk} \bdelta_{jk} = - \lambda \y_{ik} +  \lambda \y_{jk}+\bepsilon_2^{ik} - \bepsilon_2^{jk} = - \lambda (\x_i - \x_j) +\bepsilon_2^{ik} - \bepsilon_2^{jk}.
\label{eq:tdelta_diff}
\end{equation}
Adding the term $(t_{jk} - t_{ik}) \bdelta_{jk}$ to both sides of the equality to obtain
\[
t_{ik}(\bdelta_{ik} - \bdelta_{jk}) =  (t_{jk} - t_{ik}) \bdelta_{jk} - \lambda (\x_i - \x_j) +\bepsilon_2^{ik} - \bepsilon_2^{jk}.
\]
Notice that $t_{ik} \ge \norm{\y_{ik}} = \norm{\x_i - \x_k} \ge q - 2p \mu > 0$ by the primal constraint \eqref{eq:mwp_constr4} and our assumption on the duality gap. Divide the equality above by $t_{ik}$ to obtain
\begin{equation}
\bdelta_{ik} - \bdelta_{jk} =  \frac{t_{jk} - t_{ik}}{t_{ik}} \bdelta_{jk} - \frac{\lambda (\x_i - \x_j)}{t_{ik}} +\frac{\bepsilon_2^{ik} - \bepsilon_2^{jk}}{t_{ik}}.
\label{eq:delta_diff}
\end{equation}
By the perturbed complementary slackness \eqref{eq:mwcs_a1}, the primal constraint \eqref{eq:mwd_constr2} and the Cauchy-Schwarz inequality, we derive the following inequality
\[
\epsilon_1^{ik} = t_{ik} \lambda + \y_{ik}^T \bdelta_{ik} \ge t_{ik} \lambda - \norm{\y_{ik}} \cdot \norm{\bdelta_{ik}} \ge t_{ik} \lambda - \norm{\y_{ik}} \cdot \lambda,
\]
 which yields an upper bound on $t_{ik}$
 \[
 t_{ik} \le \norm{\y_{ik}} + \frac{\epsilon_1^{ik}}{\lambda}.
 \]
Combined with the primal constraint \eqref{eq:mwp_constr4} at $t_{jk}$ and the triangle inequality, we obtain the following
\begin{equation}
t_{ik} - t_{jk} \le \norm{\y_{ik}} + \frac{\epsilon_1^{ik}}{\lambda} - \norm{\y_{jk}} \le \norm{\y_{ik} - \y_{jk}} + \frac{\epsilon_1^{ik}}{\lambda} = \norm{\x_i - \x_j} + \frac{\epsilon_1^{ik}}{\lambda}.
\label{eq:t_diff}
\end{equation}
The same inequality holds for $t_{jk} - t_{ik}$ due to the symmetry of \eqref{eq:t_diff}. By \eqref{eq:delta_diff}, \eqref{eq:t_diff} and triangle inequality, the norm bound of $\bdelta_{\abrack{ik}} - \bdelta_{\abrack{jk}}$ is as follows
\begin{align*}
\norm{\bdelta_{\abrack{ik}} - \bdelta_{\abrack{jk}}}
&\le \frac{|t_{ik} - t_{jk}| \cdot \norm{\bdelta_{jk}}}{t_{ik}} + \frac{\lambda \norm{\x_i - \x_j}}{t_{ik}} +\frac{\norm{\bepsilon_2^{ik}} + \norm{\bepsilon_2^{jk}}}{t_{ik}} \quad \text{(By triangle inequality)}\\
&\le \frac{\norm{\x_i - \x_j} + \frac{\epsilon_1^{ik}}{\lambda}}{t_{ik}} \norm{\bdelta_{jk}} + \frac{\lambda \norm{\x_i - \x_j}}{t_{ik}} +\frac{\norm{\bepsilon_2^{ik}} + \norm{\bepsilon_2^{jk}}}{t_{ik}} \quad \text{(By \eqref{eq:t_diff})} \\
&\le \frac{2\lambda \norm{\x_i - \x_j}}{t_{ik}} + \frac{\norm{\bepsilon_2^{ik}} + \norm{\bepsilon_2^{jk}}}{t_{ik}} + \frac{\epsilon_1^{ik}}{t_{ik}} \quad \text{(By \eqref{eq:mwd_constr2} and \eqref{eq:tdelta_diff})}.
\end{align*}
Since $i,j \in C$ and $k \notin C$, there hold $t_{ik} \ge \norm{\y_{ik}} = \norm{\x_i - \x_k} \ge q - 2p\mu$ and $\norm{\x_i - \x_j} \le 2p \mu$. Moreover, there also hold $\epsilon_1^{ik} \le \mu$,
$\norm{\bepsilon_2^{ik}} \le \sqrt{\frac{1}{r_i r_k} \left (\sum_{l=1}^n r_l \norm{\bar \a - \a_l}^2 \mu+ 2\mu^2 \right)}$ and $\norm{\bepsilon_2^{jk}} \le \sqrt{\frac{1}{r_j r_k} \left (\sum_{l=1}^n r_l \norm{\bar \a - \a_l}^2 \mu+ 2\mu^2 \right)}$.
Hence, $\norm{\bdelta_{\abrack{ik}} - \bdelta_{\abrack{jk}}}$ is further upper bounded as follows
\begin{equation}
    \norm{\bdelta_{\abrack{ik}} - \bdelta_{\abrack{jk}}} \le \frac{4 \lambda p \mu}{q - 2p \mu} +\frac{\left (\sqrt{ \frac{1}{r_i r_k}} + \sqrt{\frac{1}{r_j r_k}}\right ) \cdot \sqrt{\sum_{l=1}^n r_l \norm{\bar \a - \a_l}^2 \mu+ 2\mu^2 }} {q - 2p \mu} + \frac{\mu}{q - 2p \mu}
    \label{eq:norm_delta_diff}
\end{equation}

\end{proof}

\underline{Bound $\norm{\bomega_i}$}
\label{sec:omega}

We restate Lemma \ref{lemma_bomega} as follows:
\begin{lemma}
For all $i\in C$, it holds
\[
\norm{\bomega_i} \le 2 \sqrt{ 2 \cdot \left (\frac{1}{r_i}\sum_{l=1}^n r_l \norm{\bar \a - \a_l}^2 + \frac{2 \mu}{r_i} + 1 \right) \cdot \left (\frac{1}{2} + \frac{1}{r_i} \left (\sum_{l=1}^n r_l (r' - r_i) \lambda \norm{\a_l} + \mu \right) \right ) \cdot \mu}.
\]
\end{lemma}

\begin{proof}
Let $i \in C$. By definition,
\[
\bomega_i = \frac{\sigma_3^i}{s_{i}} \z_{i} + \frac{1}{s_i} \bsigma_2^i.
\]
By the primal constraint \eqref{eq:mwp_constr5}, we have
\[
\norm{\z_{i}} \le \sqrt{2s_i -1} \le \sqrt{2s_i}, \quad s_i \ge \frac{1}{2},
\]
which implies
\[
\frac{\norm{\z_i}}{s_i} \le \sqrt{\frac{2}{s_i}} \le \sqrt{4} = 2, \quad \frac{1}{s_i} \le 2.
\]
Coupled with triangle inequality, these two inequalities yield
\[
\norm{\bomega_i} \le \frac{\norm{\z_{i}}}{s_{i}} \sigma_3^i + \frac{1}{s_i} \norm{\bsigma_2^i} \le 2 \sigma_3^i + 2 \norm{\bsigma_2^i}.
\]
Moreover, since
\[
\norm{\begin{pmatrix}
\bsigma_2^i\\
\sigma_3^i
\end{pmatrix}} \le \sqrt{\left (\frac{1}{r_i}\sum_{l=1}^n r_l \norm{\bar \a - \a_l}^2 + \frac{2 \mu}{r_i} + 1 \right) \cdot \left (\frac{1}{2} + \frac{1}{r_i} \left (\sum_{l=1}^n r_l (r' - r_i) \lambda \norm{\a_l} + \mu \right ) \right ) \cdot \mu}
\]
holds for any $i \in [n]$ by Lemma \ref{lemma:O(sqrt(mu))} and the duality gap,
\[
(\sigma_3^i)^2 + \norm{\bsigma_2^{i}}^2 \le \left (\frac{1}{r_i}\sum_{l=1}^n r_l \norm{\bar \a - \a_l}^2 + \frac{2 \mu}{r_i} + 1 \right) \cdot \left (\frac{1}{2} + \frac{1}{r_i} \left (\sum_{l=1}^n r_l (r' - r_i) \lambda \norm{\a_l} + \mu \right) \right ) \cdot \mu.
\]
which implies the following inequality since $(a + b)^2 \le 2 a^2 + 2b^2$
\[
(\sigma_3^i + \norm{\bsigma_2^{i}})^2  \le 2 \cdot \left (\frac{1}{r_i}\sum_{l=1}^n r_l \norm{\bar \a - \a_l}^2 + \frac{2 \mu}{r_i} + 1 \right) \cdot \left (\frac{1}{2} + \frac{1}{r_i} \left (\sum_{l=1}^n r_l (r' - r_i) \lambda \norm{\a_l} + \mu \right) \right ) \cdot \mu.
\]
Therefore, the following holds as $i \in C$ is chosen arbitrarily:
\begin{align*}
\norm{\bomega_i} &\le 2 \sigma_3^i + 2 \norm{\bsigma_2^i}\\
&\le 2 \sqrt{ 2 \cdot \left (\frac{1}{r_i}\sum_{l=1}^n r_l \norm{\bar \a - \a_l}^2 + \frac{2 \mu}{r_i} + 1 \right) \cdot \left (\frac{1}{2} + \frac{1}{r_i} \left (\sum_{l=1}^n r_l (r' - r_i) \lambda \norm{\a_l} + \mu \right) \right ) \cdot \mu}.
\end{align*}

\end{proof}

\underline{Bound CGR subgradients}
\label{sec:cgr_bound}

We restate Lemma \ref{lemma_main} as follows:
\begin{lemma}
For all $i, j \in C$ and $i < j$, there holds
    \begin{align*}
    &\norm{\q_{ij}}\\
     \le & \frac{2}{r'} \sqrt{ 2 \cdot \left (\frac{1}{r_i}\sum_{l=1}^n r_l \norm{\bar \a - \a_l}^2 + \frac{2 \mu}{r_i} + 1 \right) \cdot \left (\frac{1}{2} + \frac{1}{r_i} \left (\sum_{l=1}^n r_l (r' - r_i) \lambda \norm{\a_l} + \mu \right) \right ) \cdot \mu} \\
    + & \frac{2}{r'} \sqrt{ 2 \cdot \left (\frac{1}{r_j}\sum_{l=1}^n r_l \norm{\bar \a - \a_l}^2 + \frac{2 \mu}{r_j} + 1 \right) \cdot \left (\frac{1}{2} + \frac{1}{r_j} \left (\sum_{l=1}^n r_l (r' - r_j) \lambda \norm{\a_l} + \mu \right) \right ) \cdot \mu} \\
    +& \frac{1}{r'} \sum_{k \notin C} r_k \left ( \frac{4 \lambda p \mu}{q - 2p \mu} + \frac{2 \sqrt{\sum_{l=1}^n \norm{\bar \a - \a_l}^2 \mu + 2\mu^2}}{q - 2p \mu} + \frac{\mu}{q - 2p \mu} \right) + \lambda - r + p' \mu + \frac{2 p \mu}{r'}
  \end{align*}
\end{lemma}

\begin{proof}
Let $i, j \in C$ and $i < j$. By triangle inequality,
\[
\norm{\q_{ij}}
\le \norm{\bdelta_{ij}} + \frac{1}{r'} \cdot (\norm{\x_i - \x_j} + \norm{\bomega_i} + \norm{\bomega_j}) + \frac{1}{r'} \sum_{k \notin C} r_k\norm{\bdelta_{\abrack{ik}} - \bdelta_{\abrack{jk}}}.
\]
With the assumptions on the distance between points,
\[
\norm{\x_i - \x_j} \le 2 p \mu.
\]
By Lemma \ref{lemma_bdelta}, \ref{lemma_g_diff}, Lemma \ref{lemma_bomega} and the inequality above, we obtain
\begin{equation}
    \begin{aligned}
    &\norm{\q_{ij}}\\
     \le & \frac{2}{r'} \sqrt{ 2 \cdot \left (\frac{1}{r_i}\sum_{l=1}^n r_l \norm{\bar \a - \a_l}^2 + \frac{2 \mu}{r_i} + 1 \right) \cdot \left (\frac{1}{2} + \frac{1}{r_i} \left (\sum_{l=1}^n r_l (r' - r_i) \lambda \norm{\a_l} + \mu \right) \right ) \cdot \mu} \\
    + & \frac{2}{r'} \sqrt{ 2 \cdot \left (\frac{1}{r_j}\sum_{l=1}^n r_l \norm{\bar \a - \a_l}^2 + \frac{2 \mu}{r_j} + 1 \right) \cdot \left (\frac{1}{2} + \frac{1}{r_j} \left (\sum_{l=1}^n r_l (r' - r_j) \lambda \norm{\a_l} + \mu \right) \right ) \cdot \mu} \\
    +& \frac{1}{r'} \sum_{k \notin C} r_k \left ( \frac{4 \lambda p \mu}{q - 2p \mu} + \frac{2 \sqrt{\sum_{l=1}^n \norm{\bar \a - \a_l}^2 \mu + 2\mu^2}}{q - 2p \mu} + \frac{\mu}{q - 2p \mu} \right) + \lambda - r + p' \mu + \frac{2 p \mu}{r'}
    \end{aligned}
\end{equation}
as desired.
\end{proof}

\section{Half-moons}
\label{sec:SM_Z}
To illustrate the clustering at $\lambda = 0.0012$, we also plot the two half moons and color the clusters as shown in Figure \ref{fig:halfmoons}. Red instances belong to one cluster, and blue instances belong to another cluster. Yellow instances are assigned to clusters of singleton points.

\begin{figure}[h]
    \centering
    \includegraphics[scale=0.6]{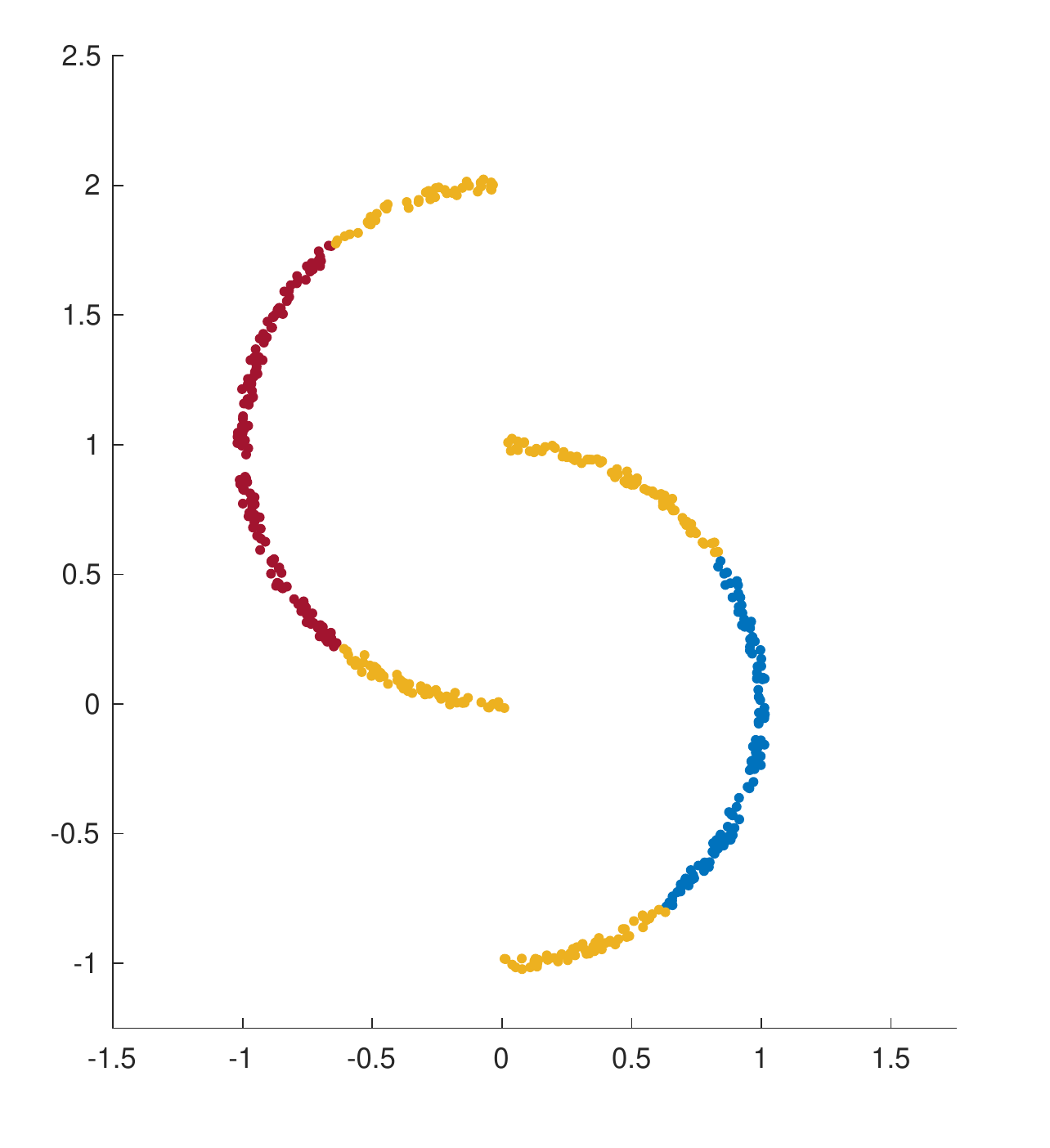}
    \caption{Labeled points with clustering at $\lambda = 0.0012$}
    \label{fig:halfmoons}
\end{figure}

\end{document}